\documentclass{article}

\usepackage[final]{neurips_2022}

\usepackage[utf8]{inputenc} 
\usepackage[T1]{fontenc}    
\usepackage{hyperref}       
\usepackage{url}            
\usepackage{booktabs}       
\usepackage{amsfonts}       
\usepackage{nicefrac}       
\usepackage{microtype}      
\usepackage{xcolor}         

\usepackage{algorithm}
\usepackage{algorithmic}

\usepackage{microtype}
\usepackage{graphicx}
\usepackage{subfigure}

\usepackage{enumitem}

\usepackage{amsmath}
\usepackage{amssymb}
\usepackage{mathtools}
\usepackage{amsthm}

\usepackage[capitalize,noabbrev]{cleveref}

\theoremstyle{plain}
\newtheorem{theorem}{Theorem}[section]

\newtheorem{lemma}[theorem]{Lemma}

\theoremstyle{definition}
\newtheorem{definition}[theorem]{Definition}

\theoremstyle{remark}

\usepackage[textsize=tiny]{todonotes}

\def\old {{\text{old}}}

\def\prox {{\text{prox}}}

\def\inner {{\text{in}}}
\def\outer {{\text{out}}}

\DeclareMathOperator*{\argmin}{arg\,min}
\def\sigmoid {{\text{sigmoid}}}
\def\R {\mathbb{R}}
\def\E {\mathbb{E}}
\def\S {\mathcal{S}}
\def\vtheta {\boldsymbol\theta}

\usepackage{epsdice}
\usepackage{bbm}

\title{Proximal Learning With Opponent-Learning Awareness}

\author{%
  Stephen Zhao \\
  University of Toronto and Vector Institute\\
  \texttt{stephen.zhao@mail.utoronto.ca} \\
   \And
  Chris Lu \\
  FLAIR, University of Oxford \\
  \texttt{christopher.lu@exeter.ox.ac.uk} \\
  \And
  Roger Grosse \\
  University of Toronto and Vector Institute\\
  \texttt{rgrosse@cs.toronto.edu} \\
  \And
  Jakob Foerster \\
  FLAIR, University of Oxford \\
  \texttt{jakob.foerster@eng.ox.ac.uk} \\
}

\begin{document}

\maketitle

\begin{abstract}

Learning With Opponent-Learning Awareness (LOLA) (\cite{foerster2018learning}) is a multi-agent reinforcement learning algorithm that typically learns reciprocity-based cooperation in partially competitive environments.
However, LOLA often fails to learn such behaviour on more complex policy spaces parameterized by neural networks, partly because the update rule is sensitive to the policy parameterization. 
This problem is especially pronounced in the opponent modeling setting, where the opponent's policy is unknown and must be inferred from observations; in such settings, LOLA is ill-specified because behaviourally equivalent opponent policies can result in non-equivalent updates.
To address this shortcoming, we reinterpret LOLA as approximating a proximal operator, and then derive a new algorithm, proximal LOLA (POLA), which uses the proximal formulation directly. 
Unlike LOLA, the POLA updates are \textit{parameterization invariant}, in the sense that when the proximal objective has a unique optimum, behaviourally equivalent policies result in behaviourally equivalent updates.
We then present practical approximations to the ideal POLA update, which we evaluate in several partially competitive environments with function approximation and opponent modeling. This empirically demonstrates that POLA achieves reciprocity-based cooperation more reliably than LOLA.

\end{abstract}

\section{Introduction}


As autonomous learning agents become more integrated into society, it is increasingly important to ensure these agents' interactions produce socially beneficial outcomes, i.e. those with high total reward. One step in this direction is ensuring agents are able to navigate social dilemmas \citep{dawes1980social}.
\citet{foerster2018learning} showed that simple applications of independent reinforcement learning (RL) to social dilemmas usually result in suboptimal outcomes from a social welfare perspective. To address this, they introduce \textit{Learning With Opponent-Learning Awareness} (LOLA), which actively shapes the learning step of other agents. LOLA with tabular policies learns \textit{tit-for-tat} (TFT), i.e. reciprocity-based cooperation, in the iterated prisoner's dilemma (IPD), a 2-agent social dilemma.

However, for the same IPD setting, we show that LOLA with policies parameterized by neural networks often converges to the Pareto suboptimal equilibrium of unconditional defection. This shows that the learning outcome for LOLA is highly dependent on \textit{policy parameterization}, one factor that makes LOLA difficult to scale to higher dimensional settings. This problem is especially pronounced in the \textit{opponent modeling} setting, where the opponent's policy is \textit{unknown} and must be inferred from observations; in such settings, LOLA is ill-specified because behaviourally equivalent opponent policies can result in non-equivalent updates and hence learning outcomes.

\begin{figure}[ht]
\vskip -0.1in
\begin{center}
\centerline{\includegraphics[scale=0.55]{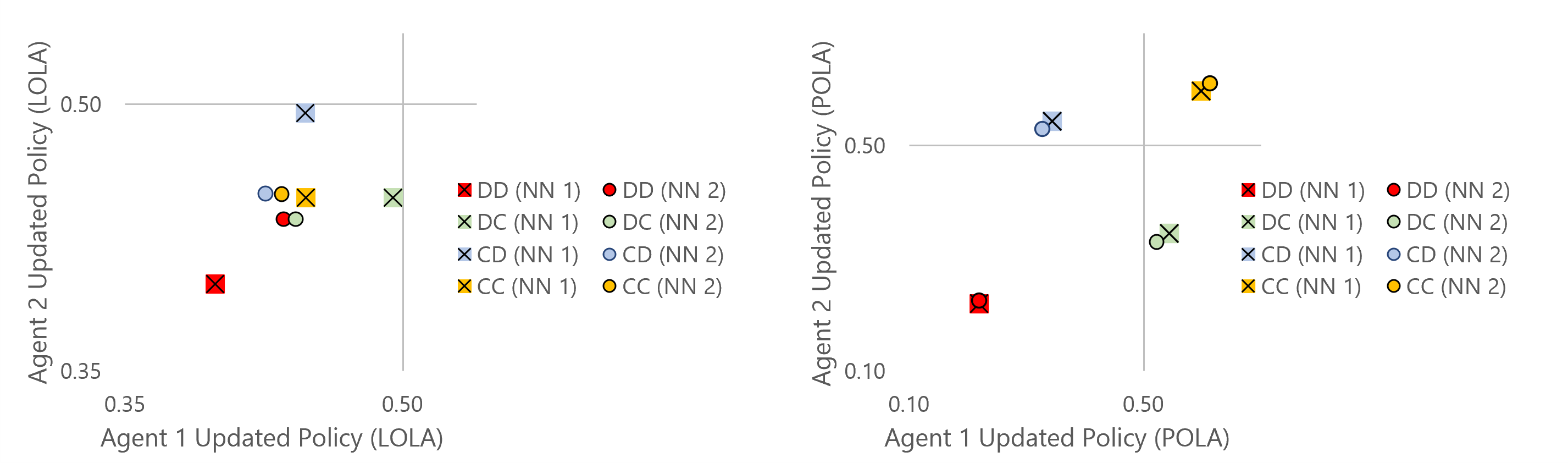}}
\caption{In the IPD (Section \ref{section:hist_one_ipd}), we initialize two neural networks using the same architecture with different weights that produce the same policy, and plot the probability of cooperation in 4 states (DD, DC, CD, CC) after a single update of LOLA and \textit{outer POLA} (Section \ref{section:pola_outer}).
Left: using LOLA; comparing circles to crosses, the different policy parameterizations result in very different updated policies despite the same starting policies and hyperparameters. 
Right: using \textit{outer POLA}; the two updated policies are much more similar.
}
\label{fig:policy_update_vis}
\end{center}
\vskip -0.25in
\end{figure}

To address these issues, we build on ideas from \textit{proximal algorithms} \citep{parikh2014proximal}, reinterpreting LOLA as a 1-step approximation of such a proximal operator, and then devise a new algorithm, \textit{proximal LOLA} (POLA), which uses this proximal formulation directly. POLA replaces the gradient updates in LOLA with proximal operators based on a \textit{proximity penalty} in policy space, which makes POLA invariant to policy parameterization. When the proximal objective has a unique optimum, behaviourally equivalent policies result in behaviourally equivalent updates. 

Solving for the exact POLA update is usually intractable, so we also provide practical algorithms for settings with and without environment rollouts that rely on approximate versions of the POLA update. We show empirically that our algorithms learn reciprocity-based cooperation with significantly higher probability than LOLA in several settings with \textit{function approximation}, even when combined with opponent modeling. 

Figure \ref{fig:policy_update_vis} illustrates our core idea: LOLA is sensitive to policy parameterization, whereas POLA is not. We compare the same policy parameterized by two different sets of weights in the same neural network architecture in the IPD, showing that our approximation to the POLA update is largely parameterization independent, while the LOLA update varies depending on parameterization.

For reproducibility, our code is available at: \url{https://github.com/Silent-Zebra/POLA}.

\subsection{Summary of Motivation, Contributions and Outline of Paper Structure}

\textbf{Motivation:} Reciprocity-based cooperation is a desired learning outcome in partially competitive environments (Section \ref{section:rec-based-coop}). LOLA typically learns such behaviour in simple policy spaces, but not on more complex policy spaces parameterized by neural networks. Our goal is to more reliably learn reciprocity-based cooperation in such policy spaces. 

\textbf{Contributions:}
\begin{itemize}[leftmargin=*]
\item We identify and demonstrate a previously unknown problem with LOLA that helps explain the aforementioned issue: LOLA is sensitive to policy parameterization (Section \ref{section:pre-cond}). 
\item To address this sensitivity, we conceptually reinterpret LOLA as approximating a proximal operator, and derive a new algorithm, \textit{ideal POLA}, which uses the proximal formulation directly (Section \ref{section:ideal_pola}). 
\item \textit{Ideal POLA} updates are invariant to policy parameterization, but usually intractable in practice. So, we develop approximations to the \textit{ideal POLA} update for use in practice (Sections \ref{section:pola_outer}, \ref{section:POLA-DiCE}). 
\item We demonstrate that our approximations more reliably learn reciprocity-based cooperation than LOLA in several partially competitive environments (Section \ref{section:hist_one_ipd}), including with function approximation and opponent modeling (Sections \ref{section:full_hist_ipd}, \ref{section:coin}).
\end{itemize}

\section{Background}

\subsection{Reinforcement Learning}

We consider the standard fully-observable reinforcement learning and Markov Game formulation:

An $N$-agent fully observable Markov Game $\mathcal{M}$ is defined as the tuple $\mathcal{M} = \langle \mathcal{S}, \mathcal{A}, \mathcal{T}, \mathcal{R}, \gamma \rangle$, where $\mathcal{S}$ is the state space, $\mathcal{A} = \{ {A}_1, ..., \mathcal{A}_N \}$ is a set of action spaces, where $\mathcal{A}_i$ for $i \in \{1,...,N\}$ denotes the action space for agent $i$, $\mathcal{T}: \mathcal{S} \times \mathcal{A}_1 \times ... \times \mathcal{A}_N \rightarrow \mathcal{P}(\mathcal{S})$ is a transition function mapping from states and actions to a probability distribution over next states, $\mathcal{R} = \{  \mathcal{R}_1, ..., \mathcal{R}_N \}$ is a set of reward functions, where $\mathcal{R}_i: \mathcal{S} \times \mathcal{A}_1 \times ... \times \mathcal{A}_N \rightarrow \mathbb{R}$ denotes the reward function for agent $i$, and $\gamma \in \mathbb{R}$ is a discount factor. Each agent acts according to a policy $\pi_{\theta^i} : \mathcal{S} \rightarrow \mathcal{P}(\mathcal{A}_i) $, with parameters $\theta^i$.

Let $\boldsymbol\theta$ be the concatenation of each of the individual $\theta^i$ (so for $N=2, \boldsymbol\theta = \{\theta^1, \theta^2\}$) and $\pi_{\vtheta}$ be the concatenation of each of the individual $\pi_{\theta^i}$. Each agent's objective is to maximize its discounted expected return:
$J^i(\pi_{\boldsymbol\theta}) = \mathbb{E}_{ s \sim P(\S), \textbf{a} \sim \pi_{\boldsymbol\theta }} \left[\sum_{t=0}^T \gamma^t r^i_{t} \right]$
where  $\textbf{a} \sim \pi_{\vtheta}$ is a joint action $\textbf{a} = \{a_1, ..., a_N\} \in {\mathcal{A}_1 \times ... \times \mathcal{A}_n}$ drawn from the policies $\pi_{\vtheta}$ given the current state $s \in \mathcal{S}$, 
$P(\S)$ is the probability distribution over states given the current policies and transition function,
$r^i_{t}$ denotes the reward achieved by agent $i$ at time step $t$ as defined by the reward function $\mathcal{R}_i$ and $T \in \mathbb{N}$ is the total time horizon or episode length. 
Throughout this paper, we assume all states are visited with non-zero probability. This is true in practice with stochastic policies, which we use throughout our experiments. 

Let $L^i(\pi_{\boldsymbol\theta}) = -J^i(\pi_{\boldsymbol\theta})$. Each agent's objective is equivalently formulated as minimizing $L^i(\pi_{\boldsymbol\theta})$. We use this as it better aligns with standard optimization frameworks.
For ease of exposition throughout this paper we consider the $N=2$ case and consider updates from the perspective of agent $1$. Updates for agent 2 follow the same structure, but with agents 1 and 2 swapped.

\subsection{Reciprocity-Based Cooperation and the IPD}
\label{section:rec-based-coop}

Reciprocity-based cooperation refers to cooperation \textit{iff} others cooperate. Unlike \textit{unconditional cooperation}, this encourages other self-interested learning agents to cooperate. One well-known example is the \textit{tit-for-tat} (TFT) strategy in the IPD \citep{axelrod1981evolution}. At each time step in the IPD, agents cooperate (C) or defect (D). Defecting always provides a higher individual reward compared to cooperating for the given time step, but both agents are better off when both cooperate than when both defect. The game is played repeatedly with a low probability of termination at each time step, which is modeled by the discount factor $\gamma$ in RL. TFT agents begin by cooperating, then cooperate \textit{iff} the opponent cooperated at the previous time step. Against a TFT agent, the best strategy for a self-interested agent is to cooperate at every time step. Other examples of reciprocity-based cooperation include contributing to a common good when others contribute (to avoid punishment), or reciprocating by collecting only resources that do not harm others (such as in Section \ref{section:coin}).
Crucially, \textit{unconditional cooperation}, for example in the single step prisoner's dilemma or against defecting opponents in the IPD, is usually a \textit{dominated strategy} and thus not a desired learning outcome. Thus, we do not compare against works that modify the learning objective and can in principle  learn \textit{unconditional cooperation} such as \citet{hughes2018inequity}, \citet{wang2018evolving}, \citet{jaques2019social}, \citet{mckee2020social}.

\subsection{Learning with Opponent-Learning Awareness (LOLA)}
\label{section:LOLA}

LOLA \citep{foerster2018learning} introduces \textit{opponent shaping} via a gradient based approach.
Instead of optimizing for $L^1 (\pi_{\theta^1}, \pi_{\theta^2}) $, i.e. the loss for agent 1 given the policy parameters $(\theta^1, \theta^2)$, agent 1 optimizes for $L^1 (\pi_{\theta^1}, \pi_{\theta^2 - \Delta \theta^2 (\theta^1)}  ) $, the loss for agent 1 after agent 2 updates its policy with one naive learning (NL) gradient
step $\Delta \theta^2 (\theta^1) = \eta \nabla_{\theta^2} L^2 (\pi_{\theta^1}, \pi_{\theta^2}) $, where $\eta$ is the opponent's learning rate. Importantly, $\Delta \theta^2 $ is treated as a function of $\theta^1$, and LOLA differentiates through the update $\Delta \theta^2 (\theta^1) $ when agent 1 optimizes for $L^1 (\pi_{\theta^1}, \pi_{\theta^2}) $. Appendix \ref{appendix:lola_update} provides more details, including pseudo-code for LOLA. \citet{foerster2018learning} also provide a policy gradient based formulation for use with environment rollouts when the loss cannot be calculated analytically, and use opponent modeling to avoid needing access to the opponent's policy parameters.

In the IPD with tabular policies, LOLA agents learn a TFT strategy,
with high probability cooperating when the other agent cooperates and defecting when the other defects \citep{foerster2018learning}. Thus, LOLA showed that \textit{with the appropriate learning rule} self-interested agents can learn policies that lead to socially optimal outcomes in the IPD.

\subsection{Proximal Point Method}
\label{background:prox:prelim}

Following \citet{parikh2014proximal}, define the proximal operator
$\prox_{\lambda f} : \mathbb{R}^n \rightarrow \mathbb{R}^n$ of a closed proper convex function $f$ as:
\begin{align*}
\prox_{\lambda f}  (v) = \argmin\limits_{x} \left(f(x) + \frac{1}{2\lambda}||x - v||_2^2 
\right)
\end{align*}
where $||\cdot||_2$ is the Euclidean (L2) norm and $\lambda$ is a hyperparameter. 
If $f$ is differentiable, its first-order approximation near $v$ is:
\begin{align*}
\hat f_v^{(1)} (x) = f(v) + \nabla f(v)^T (x-v)
\end{align*}
The proximal operator of the first-order approximation is:
\begin{align*}
\prox_{\hat f_v^{(1)}} (v) 
&= \argmin\limits_{x} \left(f(v) + \nabla f(v)^T (x-v) + \frac{1}{2\lambda}||x - v||_2^2 \right)
= v - \lambda \nabla f(v)
\end{align*}
which is a standard gradient step on the original function $f(v)$ with step size $\lambda$. 

The proximal point method starts with an iterate $x_0$, then for each time step $t \in \{1, 2, ...\}$, calculates a new iterate $x_t = \prox_{ \lambda f} (x_{t-1}) $. Gradient descent is equivalent to using the proximal point method with a first-order approximation of $f$.

\section{Proximal LOLA (POLA)}

In this section, we first explore how LOLA is sensitive to different types of changes in policy parameterization (Section \ref{section:pre-cond}). Next, we introduce \textit{ideal POLA} (Section \ref{section:ideal_pola}), a method invariant to all such changes. Lastly, we present approximations to POLA and resulting practical algorithms (Sections \ref{section:pola_outer}, \ref{section:POLA-DiCE}). 

\subsection{Sensitivity to Policy Parameterization}
\label{section:pre-cond}

LOLA is sensitive to policy parameterization, as it is defined in terms of (Euclidean) gradients, which are not invariant to smooth transformations of the parameters. 
Specifically, parameterization affects not only the convergence rate, but also which equilibrium is reached, as we illustrate with a simple toy example.
Consider the case of tabular policies with one time step of memory for the IPD.
There are five possible states: DD (both agents last defected), DC (agent 1 defected while agent 2 cooperated), CD (agent 1 cooperated, agent 2 defected), CC (both agents last cooperated), and the starting state. 
For each agent $i \in \{1,2\}$, $\theta^i \in \R^5$ parameterizes a tabular policy over these 5 states, so that $\pi_{\theta^i}(s) = \sigmoid(\theta^i)[s]$, where $v[s]$ denotes the element of vector $v$ corresponding to state $s$.
To illustrate the dependence on parameterization, we apply the invertible transformation $\textbf{Q}^i \theta^i$ where: 
\begin{align*}
\textbf{Q}^1 = 
\begin{psmallmatrix}
1 & 0 & -2 & 0 & 0 \\
0 & 1 & -2 & 0 & 0 \\
0 & 0 & 1  & 0 & 0 \\
0 & 0  & -2  & 1 & 0\\
0 & 0  & -2 & 0 & 1
\end{psmallmatrix}
,
\textbf{Q}^2 = 
\begin{psmallmatrix}
1 & -2 & 0 & 0 & 0 \\
0 & 1 & 0 & 0 & 0 \\
0 & -2 & 1  & 0 & 0 \\
0 & -2  & 0  & 1 & 0\\
0 & -2  & 0 & 0 & 1
\end{psmallmatrix}
\end{align*}
Thus, agent $i$'s policy is now $\pi_{\textbf{Q}^i \theta^i}(s) = \sigmoid (\textbf{Q}^i \theta^i)[s]$.  

\begin{definition}
For policies $\pi_{\theta^{ia}}, \pi_{\theta^{ib}}$, we say $\pi_{\theta^{ia}} = \pi_{\theta^{ib}}$ when for all $s \in \S$, $\pi_{\theta^{ia}}(s) = \pi_{\theta^{ib}}(s)$.
\end{definition}

The transformation matrices $\textbf{Q}^1 , \textbf{Q}^2 $ are non-singular and represent changes of basis. Thus, for any policy $\pi_{\theta^i}$, there exists some $\theta^{i\prime}$ such that $\pi_{\textbf{Q}^i \theta^{i\prime}} = \pi_{\theta^i}$. Despite this, LOLA fails to learn TFT in this transformed policy space (Figure \ref{fig:all_comparison}). 
We also observe this issue when comparing tabular LOLA to LOLA with policies parameterized by neural networks in Figure \ref{fig:all_comparison}; changes in policy representation materially affect LOLA. 
Relatedly, LOLA is sensitive to different policy parameterizations with the same policy representation, as Figure \ref{fig:policy_update_vis} (left) shows. 
In the next subsection, we propose an algorithm that addresses these issues.

\subsection{Ideal POLA}
\label{section:ideal_pola}

In this section, we first formalize the desired invariance property. Next, we introduce an idealized (but impractical) version of POLA, and show that it achieves the desired invariance property.

\begin{definition}
\label{def:invariance}
An update rule $u: \R^n \times \R^m \rightarrow \R^n$ that updates policy parameters $\theta^1 \in \R^n$ using auxiliary information $\theta^2 \in \R^m$
is invariant to policy parameterization when for any $\theta^{1a} \in \R^n, \theta^{1b} \in \R^n, \theta^{2a} \in \R^m, \theta^{2b} \in \R^m$ such that $\pi_{\theta^{1a}} = \pi_{\theta^{1b}}$ and $\pi_{\theta^{2a}} = \pi_{\theta^{2b}}$, $\pi_{u(\theta^{1a}, \theta^{2a})} = \pi_{u(\theta^{1b}, \theta^{2b})}$. In short, if the original policies were the same, so are the new policies (but not necessarily the new policy parameters).
\end{definition}

To achieve invariance to policy parameterization, we introduce our \textit{idealized version of proximal LOLA} (\textit{ideal POLA}). Similarly to PPO \citep{schulman2017proximal}, each player adjusts their policy to increase the probability of highly rewarded states while penalizing the distance moved in policy space. Crucially, each player also assumes the other player updates their parameters through such a proximal update. More formally, at each policy update, agent 1 solves for the following $\theta^{1\prime}$:
\begin{equation}
\label{eq:POLA_outer}
\theta^{1\prime}(\theta^1, \theta^2) = 
\argmin\limits_{\theta^{1\prime\prime}} \left(L^1 (\pi_{\theta^{1\prime\prime}}, \pi_{\theta^{2\prime}(\theta^{1\prime\prime}, \theta^2)}) + \beta_\outer D(\pi_{\theta^1} || \pi_{\theta^{1\prime\prime}}) \right)
\end{equation}
where $D(\pi_{\theta^i} || \pi_{\theta^{i\prime\prime}})$ is shorthand for a general divergence defined on policies; specific choices are given in subsequent sections. 
For $\theta^{2\prime}$ in Equation \ref{eq:POLA_outer}, we choose the following proximal update:
\begin{equation}
\label{eq:POLA_inner}
\theta^{2\prime}(\theta^{1\prime\prime}, \theta^2) =
\argmin\limits_{\theta^{2\prime\prime}} \left(L^2 (\pi_{\theta^{1\prime\prime}}, \pi_{\theta^{2\prime\prime}}  ) + \beta_\inner D(\pi_{\theta^2} || \pi_{\theta^{2\prime\prime}}) \right)
\end{equation}
For the above equations to be well defined, the $\argmin$ must be unique; we assume this in all our theoretical analysis. If different parameters can produce the same policy, or multiple different policies have the same total divergence and loss, then the $\argmin$ is non-unique. However, non-unique solutions is not an issue in practice, as our algorithms approximate $\argmin$ operations with multiple gradient updates.

\begin{theorem}
The POLA update rule $u(\theta^1, \theta^2) = \theta^{1\prime}(\theta^1, \theta^2)$, where $\theta^{1\prime}$ is defined based on Equations \ref{eq:POLA_outer} and \ref{eq:POLA_inner}, is invariant to policy parameterization. 
\end{theorem}

The proof is in Appendix \ref{appendix:invariance_proof}. In short, this follows from the fact that all terms in the $\argmin$ in Equations \ref{eq:POLA_outer} and \ref{eq:POLA_inner} are functions of policies, and not directly dependent on policy parameters.

This invariance also is a major advantage in opponent modeling settings, where agents cannot directly access the policy parameters of other agents, and must infer policies from observations. LOLA is ill-specified in such settings; the LOLA update varies depending on the assumed parameterization of opponents' policies. Conversely, POLA updates are invariant to policy parameterization, so any parameterization can be chosen for the opponent model.

To clarify the relation between LOLA and POLA, LOLA is a version of POLA that uses first-order approximations to all objectives and an L2 penalty on policy parameters rather than a divergence over policies. This follows from gradient descent being equivalent to the proximal point method with first order approximations (Section \ref{background:prox:prelim}); we provide a full proof in Appendix \ref{appendix:lola_as_prox}.

Exactly solving Equations $\ref{eq:POLA_outer}$ and $\ref{eq:POLA_inner}$ is usually intractable, so in the following Sections 
\ref{section:pola_outer} and \ref{section:POLA-DiCE}, we formulate practical algorithms that approximate the POLA update.

\subsection{Outer POLA}
\label{section:pola_outer}

To approximate \textit{ideal POLA}, we use a first order approximation to agent 2's objective, which is equivalent to agent 2 taking a gradient step (see Theorem \ref{thm:pola-lola} for details).
That is, instead of finding $\theta^{2\prime}$ via Equation \ref{eq:POLA_inner}, we use $\theta^{2\prime} = \theta^2 - \Delta \theta^2$. Agent 1 thus solves for:
\begin{equation}
\label{eq:pola_outer_only}
\theta^{1\prime}(\theta^1, \theta^2) = \argmin\limits_{\theta^{1\prime\prime}} \left(L^1  (\pi_{\theta^{1\prime\prime}}, \pi_{\theta^2 - \Delta \theta^2(\theta^{1\prime\prime})}) + \beta_\outer D(\pi_{\theta^1} || \pi_{\theta^{1\prime\prime}}) \right)
\end{equation}
Solving the $\argmin$ above exactly is usually intractable. For a practical algorithm, we differentiate through agent 2's gradient steps with unrolling, like in LOLA, and repeatedly iterate with gradient steps on agent 1's proximal objective until a fixed point is found; Algorithm \ref{algo:prox_lola_exact_outer} shows pseudo-code.
We choose $D(\pi_{\theta^1} || \pi_{\theta^{1\prime\prime}}) = \E_{s \sim U(\S)} [D_{KL}(\pi_{\theta^1}(s) || \pi_{\theta^{1\prime\prime}} (s))]$, where $U(\S)$ denotes a uniform distribution over states, as this is most analogous to an L2 distance on tabular policies, and we only test \textit{outer POLA} on settings where tabular policies can be used.

\begin{algorithm}[tb]
\caption{Outer POLA 2-agent formulation: update for agent $1$}
\label{algo:prox_lola_exact_outer}
\begin{algorithmic}
\STATE {\bfseries Input:} Policy parameters $\theta^1, \theta^2$, proximal step size $\alpha_1$, learning rate $\eta$, penalty strength $\beta_\outer$
\STATE Make copy: $\theta^{1\prime\prime} \gets \theta^1$
\REPEAT 
\STATE $ \theta^{2\prime\prime} \gets \theta^{2} - \eta \nabla_{\theta^{2}} L^2(\pi_{\theta^{1\prime\prime}}, \pi_{\theta^{2}})  $
\STATE $\theta^{1\prime\prime} \gets \theta^{1\prime\prime} - \alpha_1 \nabla_{\theta^{1\prime\prime}} (L^1(\pi_{\theta^{1\prime\prime}}, \pi_{\theta^{2\prime\prime}}) + \beta_{\outer} (\E_{s \sim U(\S)} [D_{KL}(\pi_{\theta^1}(s) || \pi_{\theta^{1\prime\prime}} (s))]) )$ 
\UNTIL {$\theta^{1\prime\prime}$ has converged to a fixed point}
\STATE {\bfseries Output:} $\theta^{1\prime\prime}$
\end{algorithmic}
\end{algorithm}

\subsection{POLA-DiCE}
\label{section:POLA-DiCE}

\textit{Outer POLA} assumes we can calculate the exact loss, but often we need to estimate the loss with samples from environment rollouts. For these cases, 
we introduce a policy gradient version of POLA adapted to work with DiCE \citep{foerster2018dice}. DiCE is a sample-based estimator that makes it easy to estimate higher order derivatives using backprop. More details about DiCE and its combination with LOLA are in the Appendix (\ref{appendix:dice}, \ref{appendix:loaded_dice}) and in ~\cite{foerster2018dice}.

POLA-DiCE approximates the $\argmin$ in \textit{ideal POLA} (Equations \ref{eq:POLA_outer} and \ref{eq:POLA_inner}) with a fixed number of gradient steps on the proximal objectives, where steps on the outer objective (\ref{eq:POLA_outer}) differentiate through the unrolled steps on the inner objective (\ref{eq:POLA_inner}). We choose $D(\pi_{\theta^i} || \pi_{\theta^{i\prime\prime}}) = \E_{s \sim \S} [D_{KL}(\pi_{\theta^i}(s) || \pi_{\theta^{i\prime\prime}} (s)) ]$, and approximate the expectation under the true state visitation with an average over states visited during rollouts. 
$s_{\leq T}$ denotes the states from a set of rollouts with $T$ time steps: $s_{\leq T} \triangleq \{s_t : t \in \{1,...,T\}  \}$, where $s_t$ is the state at time step $t$.
$D(\pi_{\theta^i}, \pi_{\theta^{i\prime\prime}} | s_{\leq T}) \triangleq \frac{1}{|s_{\leq T}|} \sum\limits_{s \in s_{\leq T}} [D_{KL}(\pi_{\theta^i}(s) || \pi_{\theta^{i\prime\prime}} (s)) ]$ denotes our sample based approximation to $\E_{s \sim \S} [D_{KL}(\pi_{\theta^i}(s) || \pi_{\theta^{i\prime\prime}} (s)) ]$.
As in LOLA-DiCE, rollouts are done in a simulator, and used in the DiCE loss $\mathcal{L}^i_{\epsdice{2}(\pi_{\theta^{1}}, \pi_{\theta^{2}} )}$. For our experiments, we assume full knowledge of transition dynamics, but in principle an environment model could be used instead.
Algorithm \ref{algo:pola_dice} provides pseudo-code for POLA-DiCE.

\begin{algorithm}
\caption{POLA-DiCE 2-agent formulation: update for agent $1$}
\label{algo:pola_dice}
\begin{algorithmic}
\STATE {\bfseries Input:} Policy parameters $\theta^1, \theta^2$, learning rates $\alpha_1, \alpha_2$, penalty hyperparameters $\beta_\inner, \beta_\outer$, number of outer steps $M$ and inner steps $K$
\STATE Initialize: $\theta^{1\prime\prime} \gets \theta^1$
\FOR {$m$ in $1...M$}
\STATE Initialize: $\theta^{2\prime\prime} \gets \theta^2$
\FOR {$k$ in $1...K$}
\STATE Rollout trajectories with states $s^{\inner}_{\leq T}$
using policies $(\pi_{\theta^{1\prime\prime}}, \pi_{\theta^{2\prime\prime}})$
\STATE $\theta^{2\prime\prime} \gets \theta^{2\prime\prime} - \alpha_2 \nabla_{\theta^{2\prime\prime} } 
( \mathcal{L}_{\epsdice{2}(\pi_{\theta^{1\prime\prime}}, \pi_{\theta^{2\prime\prime}})}^2  + \beta_\inner D(\pi_{\theta^2}, \pi_{\theta^{2\prime\prime}} | s^{\text{in}}_{\leq T}) )
$
\ENDFOR
\STATE Rollout trajectories with states $s^{\outer}_{\leq T}$
using policies $(\pi_{\theta^{1\prime\prime}}, \pi_{\theta^{2\prime\prime}})$
\STATE $\theta^{1\prime\prime} \gets \theta^{1\prime\prime} - \alpha_1 \nabla_{\theta^{1\prime\prime}} (\mathcal{L}^1_{\epsdice{2}(\pi_{\theta^{1\prime\prime}}, \pi_{\theta^{2\prime\prime}})} + \beta_\outer D(\pi_{\theta^1}, \pi_{\theta^{1\prime\prime}} | s^{\text{out}}_{\leq T}))$
\ENDFOR
\STATE {\bfseries Output:} $\theta^{1\prime\prime}$
\end{algorithmic}
\end{algorithm}

For sufficiently large numbers of inner steps $K$ and outer steps $M$, and sufficiently small learning rates, POLA-DiCE iterates until a fixed point is found. Unfortunately, iterating to convergence often requires a very large amount of rollouts (and memory for the inner steps).
When $M = 1$ and $\beta_\inner, \beta_\outer = 0$,
POLA-DiCE is equivalent to LOLA-DiCE \citep{foerster2018dice}.

Without opponent modeling, LOLA-DiCE and POLA-DiCE directly access the other agent's policy parameters, using those in simulator rollouts for policy updates. In the opponent modeling (OM) setting, agents can only access actions taken by the other agent, from real environment rollouts. For policy updates, agents must use learned policy models in simulator rollouts. We learn policy models by treating observed actions as targets in a supervised classification setting, as in behaviour cloning \citep{ross2011no}. Figure \ref{fig:om_graphic} illustrates the process at each time step; we keep policy models from previous iterations and update incrementally on newly observed actions.

\begin{figure}[ht]
\vskip -0.15in
\begin{center}
\centerline{\includegraphics[scale=0.5]{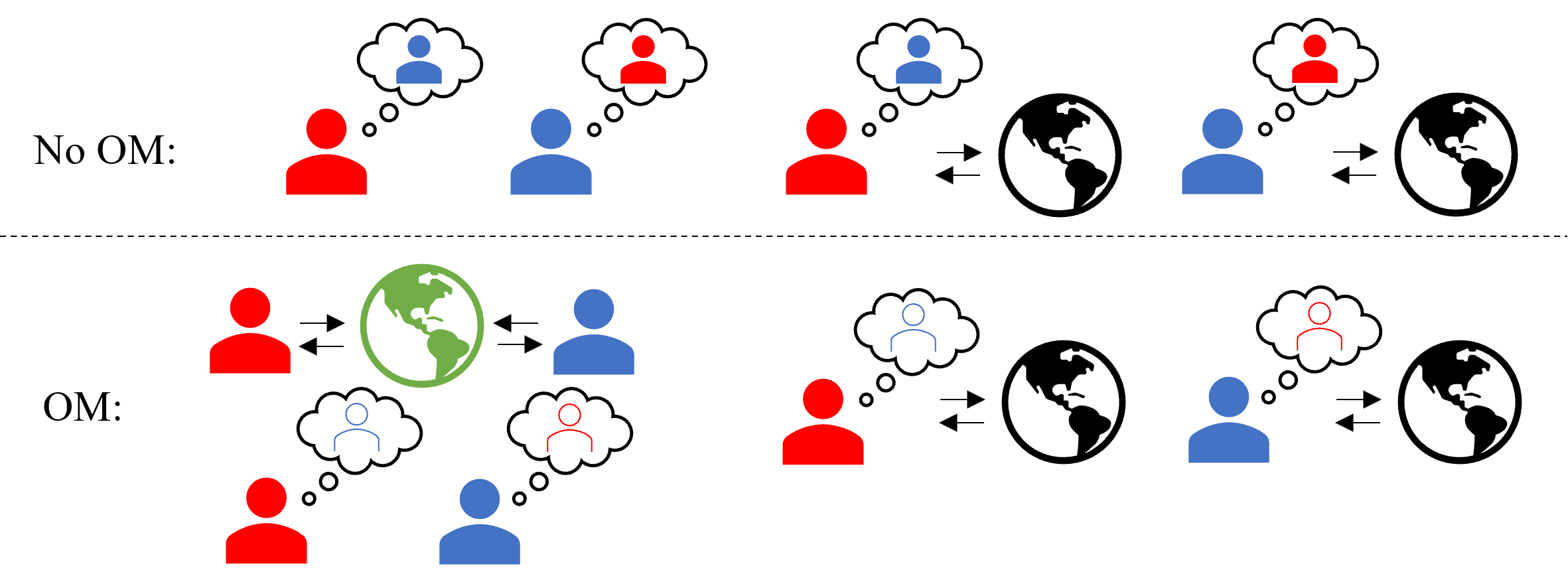}}
\caption{Illustration of the training process at each time step for LOLA-DiCE and POLA-DiCE. Without opponent modeling, each agent directly uses the other agent's policy parameters in simulator rollouts for policy updates. With opponent modeling, each agent first learns a policy model of the opponent based on observed actions from real environment rollouts. Agents then use the learned models in simulator rollouts for policy updates. We assume known dynamics, so no environment model needs to be learned for simulator rollouts.}
\label{fig:om_graphic}
\end{center}
\vskip -0.3in
\end{figure}

\section{Experiments}

To investigate how useful our approximations to \textit{ideal POLA} are, we run experiments to answer the following questions:
1) Does \textit{outer POLA} learn reciprocity-based cooperation more consistently than LOLA, across a variety of policy parameterizations, in the IPD with one-step memory (Section \ref{section:hist_one_ipd})? 
2) Does POLA-DiCE learn reciprocity-based cooperation more consistently than LOLA-DiCE in settings that require function approximation and rollouts (IPD with full history (Section \ref{section:full_hist_ipd}) and coin game (Section \ref{section:coin}))? If so, do these results still hold when using opponent modeling?

\subsection{One-Step Memory IPD}
\label{section:hist_one_ipd}

With one-step memory, the observations are the actions by both agents at the previous time step. 
In this setting, we can use the exact loss in gradient updates (see Appendix \ref{appendix:exact_loss_ipd} for details).

To investigate behaviour across settings that vary the difficulty of learning reciprocity-based cooperation, we introduce a \textit{cooperation factor} $f \in \R$, which determines
the reward for cooperating relative to defecting. At each time step, let $c$ be the number of agents who cooperated. Each agent's reward is $c * f / 2 - \mathbbm{1}[\text{agent contributed}]$. This is a specific instance of the \textit{contribution game} from \citet{barbosa2020emergence}, with two agents. $1 < f < 2$ satisfies social dilemma characteristics, where defecting always provides higher individual reward, but two cooperators are both better off than two defectors.
For more details on the problem setup, see Appendices \ref{appendix:cf_details} and \ref{appendix:ipd-one-funcapprox-state}.

Figure \ref{fig:all_comparison} compares LOLA and \textit{outer POLA} (Section \ref{section:pola_outer}) using various $f$ and policy parameterizations. To provide a succinct graphical representation of how well agents learn reciprocity-based cooperation, we test how often agents \textit{find TFT}. We consider agents cooperating with each other with high probability, achieving average reward $> 80\%$ of the socially optimal, but both cooperating with probability $< 0.65$ if the opponent last defected, to have \textit{found TFT}. These thresholds are somewhat arbitrary; Appendix \ref{appendix:ipd-one-detailed_results} provides detailed policy probabilities that support our conclusions without such thresholds.

We reproduce 
\citet{foerster2018learning}'s result that naive learning converges to unconditional defection, while LOLA using tabular policies finds TFT.
However, changes in policy parameterization greatly hinder LOLA's ability to find TFT. In contrast, \textit{outer POLA} finds TFT even with function approximation, and finds TFT significantly more often than LOLA in the pre-conditioned setting described in Section \ref{section:pre-cond}. Appendix \ref{appendix:ipd1_hparams} further discusses hyperparameter settings. 

\begin{figure}[ht]
\vskip -0.15in
\begin{center}
\centerline{\includegraphics[scale=0.6]{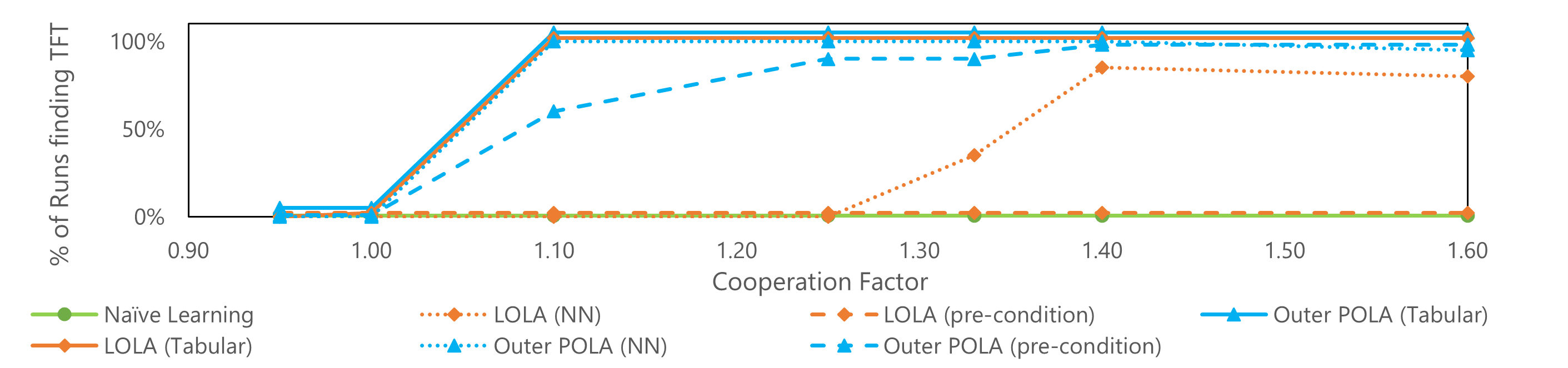}}
\caption{
Comparison of LOLA and \textit{outer POLA} 
in the one-step memory IPD with various $f$, plotting the percentage of 20 runs that \textit{find TFT}. 
``NN'' denotes policies parameterized by neural networks. ``Pre-condition'' denotes the setting in Section \ref{section:pre-cond}.
Given hyperparameter tuning, tabular LOLA always finds TFT (for $f > 1$), whereas LOLA with function approximation fails on lower $f$. In the pre-conditioned setting, LOLA completely fails to find TFT. 
In contrast, \textit{outer POLA} 
finds TFT regardless of whether the policy is tabular or a neural network, and retains most of its performance in the pre-conditioned setting.
As sanity checks, naive learning ($\eta = 0$) always fails to find TFT, and all algorithms always defect for $f < 1$. }
\label{fig:all_comparison}
\end{center}
\vskip -0.25in
\end{figure}

Qualitatively, Figure \ref{fig:policy_update_vis} shows that in the IPD with one-step memory, $f = 1.33$, and neural network parameterized policies, \textit{outer POLA} closely approximates the invariance provided by \textit{ideal POLA}.

To demonstrate that POLA allows for any choice of opponent model, Appendix \ref{appendix:ipd-om} provides further experiments in the IPD with opponent modeling, using a version of POLA similar to POLA-DiCE.

\subsection{Full History IPD}
\label{section:full_hist_ipd}

Next, we relax the assumption that agents are limited to one-step memory and instead consider policies that condition on the entire history, which makes using function approximation and rollouts necessary. We parameterize policies using GRUs \citep{cho2014properties} and test whether POLA-DiCE still learns reciprocity-based cooperation within this much larger 
policy space.
We use $f = 1.33$ for the reward structure. For more details on the problem setting, policy parameterization, and hyperparameters, see Appendix \ref{appendix:full_hist_ipd}.

Figure \ref{fig:gru_ipd} shows results in this setting. LOLA agents sometimes learn reciprocity-based cooperation but often learn to always defect, achieving low scores against each other on average. POLA agents learn reciprocity-based cooperation much more consistently, almost always cooperating with each other but defecting with high probability if the opponent always defects.
Furthermore, opponent modeling works well with POLA; POLA agents behave similarly using opponent modeling instead of accessing policy parameters directly.

\begin{figure}[ht]
\vskip -0.15in
\begin{center}
\centerline{\includegraphics[scale=0.4]{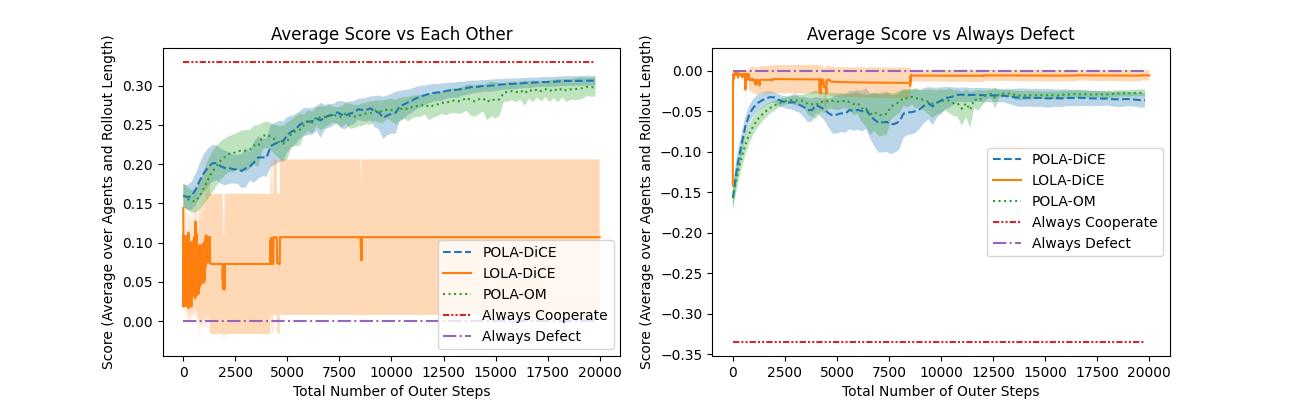}}
\caption{Comparison of LOLA-DiCE and POLA-DiCE on the IPD with GRU parameterized policies that condition on the full action history.
Left: POLA-DiCE agents learn to cooperate with each other with high probability, achieving close to the socially optimal reward (always cooperate), whereas LOLA agents cooperate with each other much less consistently, often learning to always defect.
Right: we test the learned policies against a hard-coded rule that always defects. POLA-DiCE agents defect against such an opponent with high probability, achieving close to the optimal reward of 0, showing POLA agents cooperate only when the other agent reciprocates.
POLA-OM (POLA-DiCE with opponent modeling) agents show similar behaviour to POLA-DiCE agents.
All results are averaged over 10 random seeds with 95\% confidence intervals shown.
}
\label{fig:gru_ipd}
\end{center}
\vskip -0.3in
\end{figure}

\subsection{Coin Game}
\label{section:coin}

To investigate the scalability of POLA-DiCE under higher dimensional observations and function approximation, we test LOLA-DiCE and POLA-DiCE in the coin game setting from \citet{lerer2017maintaining}. The coin game consists of a 3x3 grid in which two agents, red and blue, move around and collect coins. There is always one coin, coloured either red or blue, which spawns with the other colour after being collected. Collecting any coin grants a reward of 1; collecting a coin with colour different from the collecting agent gives the other agent -2 reward. The coin game embeds a temporally extended social dilemma; if agents defect by trying to pick up all coins, both agents get 0 total reward in expectation, whereas if agents cooperate by only picking up coins of their own colour, they achieve positive expected average reward (maximum: $\frac 1 3$ per time step). Appendix \ref{appendix:coin} provides more detail. We again parameterize the agents' policies with GRUs \citep{cho2014properties}.

Figure \ref{fig:coin} shows that
POLA-DiCE agents learn to cooperate with higher probability than LOLA-DiCE agents,\footnote{See Appendix \ref{appendix:coin} for why the LOLA results cannot be directly compared with \citet{foerster2018learning}} picking up a larger proportion of their own coins, and again do not naively cooperate. POLA-DiCE agents with opponent modeling also learn reciprocity-based cooperation, but slightly less so than with direct access to policy parameters, likely due to the noise from opponent modeling. 

\begin{figure}[ht]
\vskip -0.15in
\begin{center}
\centerline{\includegraphics[scale=0.4]{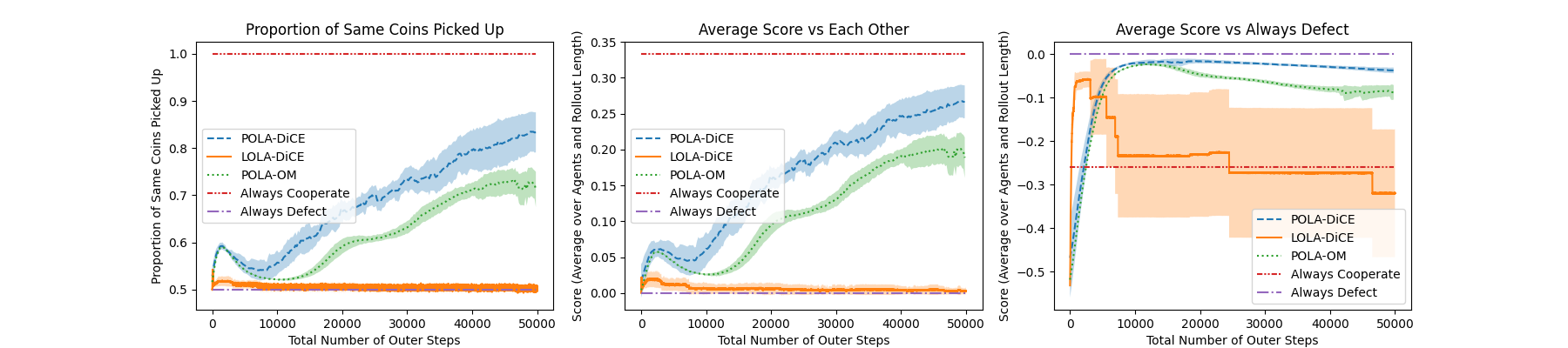}}
\caption{Comparison of LOLA-DiCE and POLA-DiCE on the coin game. Left: number of coins collected where the coin's colour matched the agent's colour, divided by the total number of coins collected. POLA agents cooperate more than LOLA agents, collecting a larger proportion of their own coloured coins. Middle: average score for the agents against each other. POLA-DiCE agents achieve close to the maximum cooperative reward.
Right: we test the learned policies against a hard-coded rule that always \textit{defects} by taking the shortest path to the coin regardless of colour. POLA agents defect, achieving scores close to the optimal of 0 against such an opponent, again showing POLA agents cooperate based on reciprocity. LOLA agents quickly learn to always defect; subsequent training is highly unstable.}
\label{fig:coin}
\end{center}
\vskip -0.3in
\end{figure}

\section{Related Work}
\label{section:relatedwork}

POLA-DiCE (Section \ref{section:POLA-DiCE}) can alternatively be motivated as replacing the policy gradient update inherent in LOLA-DiCE with a proximal-style update like that of TRPO \citep{schulman2015trust} or the adaptive KL-penalty version of PPO \citep{schulman2017proximal}. 

There are other extensions to LOLA that address problems orthogonal to policy parameterization sensitivity. 
In concurrent work, \citet{willi2022cola} highlight that LOLA is inconsistent in self-play;
the update each player assumes the opponent makes 
is different from their actual update. They propose COLA,
learning update rules $\Delta \theta^1$ and $\Delta \theta^2$ that are consistent with each other. However, their notion of consistency is gradient-based and thus not invariant to policy parameterization, whereas POLA is not consistent. SOS \citep{letcher2018stable}, which combines the advantages of LOLA and LookAhead \citep{zhang2010multi}, is also not invariant to policy parameterization. 

\citet{al2018continuous} formulate a meta-RL framework with a policy gradient theorem, and apply PPO as their optimization algorithm; this is one instance of combining proximal style operators with reinforcement learning style policy gradients for optimization with higher order gradients. However, they focus on adaptation rather than explicitly modeling and shaping opponent behaviour. 

MATRL \citep{wen2021game} sets up a metagame between the original policies and new policies calculated with independent TRPO updates for each agent,
solving for a Nash equilibrium in the metagame at each step. MATRL finds a best response without actively shaping its opponents' learning, so it is similar to LookAhead \citep{zhang2010multi} (which results in unconditional defection) rather than LOLA. 

M-FOS \citep{lu2022model} sets up a metagame where each meta-step is a full environment rollout, and the meta-reward at each step is the cumulative discounted return of the rollout. They use model-free optimization methods to learn meta-policies that shape the opponent's behaviour over meta-steps. We do not compare against it as it is concurrent work without published code at the time of writing.

\citet{badjatiya2021status} introduce \textit{status-quo loss}, in which agents imagine the current situation being repeated for several steps; combined with the usual RL objective, this leads to cooperation in social dilemmas. However, their method learns to defect in the one-step memory IPD in the states DC and CD, and thus does not learn reciprocity-based cooperation. 

\citet{zhao2021provably} provide a detailed numerical analysis of the function approximation setting for two-player zero-sum Markov Games.
\citet{fiez2021minimax} considers 
a proximal method for minimax optimization. In contrast to both, we consider the general-sum setting.

\section{Limitations}
\label{section:limitations}

\textit{Outer POLA} and POLA-DiCE approximate \textit{ideal POLA}, so are not completely invariant to policy parameterization; Figures \ref{fig:policy_update_vis} and \ref{fig:all_comparison} show how close the approximation is in the one-step memory IPD. POLA introduces additional hyperparameters compared to LOLA (e.g. $\beta_\inner, \beta_\outer$ and number of outer steps for POLA-DiCE). POLA usually requires many rollouts for each update; future work could explore ways to mitigate this. For example, in Appendix \ref{appendix:pola_dice_repeat_train} we formulate a version of POLA-DiCE closer to PPO \citep{schulman2017proximal}, that repeatedly trains on samples for improved sample efficiency.

\section{Conclusion and Future Work}

Motivated by making LOLA invariant to policy parameterization, we introduced \textit{ideal POLA}. Empirically, we demonstrated that practical approximations to \textit{ideal POLA} learn reciprocity-based cooperation significantly more often than LOLA in several social dilemma settings. 

Future work could explore and test: policy parameterization invariant versions of LOLA extensions like COLA \citep{willi2022cola} and SOS \citep{letcher2018stable}; POLA formulations with improved sample efficiency (e.g. Appendix \ref{appendix:pola_dice_repeat_train}); the $N$-agent version of POLA (Appendices \ref{appendix:pola_N}, \ref{appendix:pola_N_dice}) in $N$-agent versions of the IPD \citep{barbosa2020emergence} and larger environments such as Harvest and Cleanup \citep{hughes2018inequity}; POLA formulations with adaptive penalty parameters $\beta_\inner, \beta_\outer$; connections to Mirror Learning \citep{kuba2022mirror}; adapting approximations to proximal point methods such as the extra-gradient method \citep{mokhtari2020unified} to work with POLA.

We are excited and hopeful that this work opens the door to scaling LOLA and related algorithms to practical settings that require function approximation and opponent modeling.

\section*{Acknowledgements}


Resources used in this research were provided, in part, by the Province of Ontario, the Government of Canada, and companies sponsoring the Vector Institute. 
The Foerster Lab for AI Research (FLAIR) is grateful to the Cooperative AI Foundation for their support. 
We thank the anonymous NeurIPS reviewers for helpful comments on earlier versions of this paper.

\bibliographystyle{plainnat}
\bibliography{refs}

\section*{Checklist}

\begin{enumerate}

\item For all authors...
\begin{enumerate}
  \item Do the main claims made in the abstract and introduction accurately reflect the paper's contributions and scope?
    \answerYes{}
  \item Did you describe the limitations of your work?
    \answerYes{Section \ref{section:limitations}. }
  \item Did you discuss any potential negative societal impacts of your work?
    \answerYes{Appendix \ref{section:impact}.}
  \item Have you read the ethics review guidelines and ensured that your paper conforms to them?
    \answerYes{}
\end{enumerate}

\item If you are including theoretical results...
\begin{enumerate}
  \item Did you state the full set of assumptions of all theoretical results?
    \answerYes{some assumptions like non-zero state visitation probability are stated earlier (e.g. in the background section) and apply throughout.}
        \item Did you include complete proofs of all theoretical results?
    \answerYes{in the Appendix, with very brief proof sketches in the main paper.} 
\end{enumerate}

\item If you ran experiments...
\begin{enumerate}
  \item Did you include the code, data, and instructions needed to reproduce the main experimental results (either in the supplemental material or as a URL)?
    \answerYes{Github repo with code and instructions provided in intro (\url{https://github.com/Silent-Zebra/POLA}).}
  \item Did you specify all the training details (e.g., data splits, hyperparameters, how they were chosen)?
    \answerYes{Either in the main paper or in the Appendix (mostly Appendix \ref{appendix:expmt_hparams}).}
        \item Did you report error bars (e.g., with respect to the random seed after running experiments multiple times)?
    \answerYes{Figures \ref{fig:gru_ipd} and \ref{fig:coin}.}
        \item Did you include the total amount of compute and the type of resources used (e.g., type of GPUs, internal cluster, or cloud provider)?
    \answerYes{Appendix \ref{appendix:compute}.}
\end{enumerate}

\item If you are using existing assets (e.g., code, data, models) or curating/releasing new assets...
\begin{enumerate}
  \item If your work uses existing assets, did you cite the creators?
    \answerYes{Details in Appendix \ref{appendix:code_details}.}
  \item Did you mention the license of the assets?
    \answerYes{Details in Appendix \ref{appendix:code_details}.}
  \item Did you include any new assets either in the supplemental material or as a URL?
    \answerYes{Github repo with code provided (\url{https://github.com/Silent-Zebra/POLA}).}
  \item Did you discuss whether and how consent was obtained from people whose data you're using/curating?
    \answerNA{The only ``data'' is rollouts from environments.}
  \item Did you discuss whether the data you are using/curating contains personally identifiable information or offensive content?
    \answerNA{The only ``data'' is rollouts from environments.}
\end{enumerate}

\item If you used crowdsourcing or conducted research with human subjects...
\begin{enumerate}
  \item Did you include the full text of instructions given to participants and screenshots, if applicable?
    \answerNA{}
  \item Did you describe any potential participant risks, with links to Institutional Review Board (IRB) approvals, if applicable?
    \answerNA{}
  \item Did you include the estimated hourly wage paid to participants and the total amount spent on participant compensation?
    \answerNA{}
\end{enumerate}

\end{enumerate}


\pagebreak

\appendix

\section{Appendix: Additional Background, Derivations, and Algorithm Details}

\subsection{LOLA with Direct Update}
\label{appendix:lola_update}

Throughout this paper, we interpret LOLA as directly differentiating through $L(\pi_{\theta^{1}}, \pi_{\theta^2 - \Delta \theta^2(\theta^1)})$; we provide an algorithm box in Algorithm \ref{algo:LOLA_update}.
We use this formulation of LOLA since it avoids using a first-order approximation around the objective as in \citet{foerster2018learning}, and forms the conceptual basis for LOLA-DiCE \citep{foerster2018dice}. In our experience, we find this formulation generally provides better and more consistent results, and is easier to implement.

\begin{algorithm}
\caption{LOLA direct update 2-agent formulation: exact gradient update for agent $1$}
\label{algo:LOLA_update}
\begin{algorithmic}
\STATE {\bfseries Input:} Policy parameters $\boldsymbol\theta = \{\theta^1, \theta^2\}$, learning rates $\alpha$, $\eta$
\STATE Initialize: $\boldsymbol\theta' \gets \boldsymbol\theta$
\STATE $ \theta^{2\prime} \gets \theta^{2\prime} - \eta \nabla_{\theta^{2\prime}} L^2(\pi_{\boldsymbol\theta'}) $ // Preserve Computation Graph 
\STATE $\theta^{1\prime} \gets \theta^1 - \alpha \nabla_{\theta_1'} L^1(\pi_{\boldsymbol\theta'}) $ // Differentiate through agent 2's update
\STATE {\bfseries Output:} $\theta^{1\prime}$
\end{algorithmic}
\end{algorithm}

\subsection{POLA Invariance Proof}
\label{appendix:invariance_proof}

\begin{proof}
Consider arbitrary $(\theta^{1a}, \theta^{1b}, \theta^{2a}, \theta^{2b} )$ such that $\pi_{\theta^{1a}} = \pi_{\theta^{1b}}$ and $\pi_{\theta^{2a}} = \pi_{\theta^{2b}}$, and consider $u(\theta^{1a}, \theta^{2a})$ and $u(\theta^{1b}, \theta^{2b})$. 
$$u(\theta^{1a}, \theta^{2a}) = 
\argmin\limits_{\theta^{1\prime\prime}} \left(L^1 (\pi_{\theta^{1\prime\prime}}, \pi_{\theta^{2\prime}(\theta^{1\prime\prime}, \theta^{2a})}) + \beta_\outer D(\pi_{\theta^{1a}} || \pi_{\theta^{1\prime\prime}}) \right)$$
$$u(\theta^{1b}, \theta^{2b}) = 
\argmin\limits_{\theta^{1\prime\prime}} \left(L^1 (\pi_{\theta^{1\prime\prime}}, \pi_{\theta^{2\prime}(\theta^{1\prime\prime}, \theta^{2b})}) + \beta_\outer D(\pi_{\theta^{1b}} || \pi_{\theta^{1\prime\prime}}) \right)$$

Recall again that we assume all $\argmin$ to be unique. From Equation \ref{eq:POLA_inner}:
$$\theta^{2\prime}(\theta^{1\prime\prime}, \theta^{2a}) =
\argmin\limits_{\theta^{2\prime\prime}} \left(L^2 (\pi_{\theta^{1\prime\prime}}, \pi_{\theta^{2\prime\prime}}  ) + \beta_\inner D(\pi_{\theta^{2a}} || \pi_{\theta^{2\prime\prime}}) \right)$$
$$\theta^{2\prime}(\theta^{1\prime\prime}, \theta^{2b}) =
\argmin\limits_{\theta^{2\prime\prime}} \left(L^2 (\pi_{\theta^{1\prime\prime}}, \pi_{\theta^{2\prime\prime}} ) + \beta_\inner D(\pi_{\theta^{2b}} || \pi_{\theta^{2\prime\prime}}) \right)$$
Since $\pi_{\theta^{2a}} = \pi_{\theta^{2b}}$, we get $\pi_{\theta^{2\prime}(\theta^{1\prime\prime},  \theta^{2a})} = \pi_{\theta^{2\prime}(\theta^{1\prime\prime}, \theta^{2b})}$ because:
$$
\argmin\limits_{\theta^{2\prime\prime}} \left(L^2 (\pi_{\theta^{1\prime\prime}}, \pi_{\theta^{2\prime\prime}}  ) + \beta_\inner D(\pi_{\theta^{2a}} || \pi_{\theta^{2\prime\prime}}) \right) = 
\argmin\limits_{\theta^{2\prime\prime}} \left(L^2 (\pi_{\theta^{1\prime\prime}}, \pi_{\theta^{2\prime\prime}} ) + \beta_\inner D(\pi_{\theta^{2b}} || \pi_{\theta^{2\prime\prime}}) \right) $$ 
Combining this with $\pi_{\theta^{1a}} = \pi_{\theta^{1b}}$, we get $u(\theta^{1a}, \theta^{2a}) = u(\theta^{1b}, \theta^{2b})$ because:
\begin{align*}
\argmin\limits_{\theta^{1\prime\prime}} \left(L^1 (\pi_{\theta^{1\prime\prime}}, \pi_{\theta^{2\prime}(\theta^{1\prime\prime}, \theta^{2a})}) + \beta_\outer D(\pi_{\theta^{1a}} || \pi_{\theta^{1\prime\prime}}) \right) = \\
\argmin\limits_{\theta^{1\prime\prime}} \left(L^1 (\pi_{\theta^{1\prime\prime}}, \pi_{\theta^{2\prime}(\theta^{1\prime\prime}, \theta^{2b})}) + \beta_\outer D(\pi_{\theta^{1b}} || \pi_{\theta^{1\prime\prime}}) \right)
\end{align*}
Thus, $\pi_{u(\theta^{1a}, \theta^{2a})} = \pi_{u(\theta^{1b}, \theta^{2b})} $, so $u$ is invariant to policy parameterization.
\end{proof}

\subsection{Proof of Connection Between POLA and LOLA}
\label{appendix:lola_as_prox}

Here we show that LOLA is a version of POLA that uses first-order approximations to all objectives and an L2 penalty on policy parameters rather than a divergence over policies. We assume the loss is differentiable so first-order approximations are well-defined, and that all $\argmin$ are unique.

\begin{lemma}
\label{lemma:lola_is_prox}
LOLA is equivalent to applying the proximal operator on a first-order approximation (around $\theta^1$) of the modified objective $L^1 (\pi_{\theta^1}, \pi_{\theta^2 - \Delta \theta^2(\theta^1)} ) $.
\end{lemma}
\begin{proof} 
Let $f^1(\theta^1) = L^1 (\pi_{\theta^1}, \pi_{\theta^2 - \Delta \theta^2(\theta^1)} ) $.

Consider the first-order Taylor approximation  $\hat f_{\theta^1}^{1(1)}$ of the modified objective function (around $\theta^1$):\footnote{Note that this is different from the first-order approximation used in the original LOLA paper, as that approximation is a Taylor series expansion of $L^1 (\pi_{\theta^1}, \pi_{\theta^2}  )$, whereas here we use a Taylor series expansion of $L^1 (\pi_{\theta^1}, \pi_{\theta^2 - \Delta \theta^2(\theta^1)})$}
\begin{align*}
\hat f_{\theta^1}^{1(1)} ({\theta^{1\prime\prime}}) &= f^1({\theta^1}) + \nabla_{\theta^1} f^1({\theta^1})^T ({\theta^{1\prime\prime}}-{\theta^1}) 
\end{align*}

Applying the proximal operator to the above, and using the results from Section \ref{background:prox:prelim} (with $x = \theta^{1\prime\prime}$ and $v = \theta^1$), we have:
\begin{align*}
\prox_{\hat f_{\theta^1}^{1(1)}} (\theta^1) &=\argmin\limits_{\theta^{1\prime\prime}} \left(\hat f_{\theta^1}^{1(1)} ({\theta^{1\prime\prime}})+ \frac{1}{2\lambda} || \theta^1 - \theta^{1\prime\prime}||_2^2 \right) \\ 
&= \theta^1 - \lambda \nabla_{\theta^1} f^1(\theta^1) = \theta^1 - \lambda \nabla_{\theta^1} L^1 (\pi_{\theta^1}, \pi_{\theta^2 - \Delta \theta^2(\theta^1)}  )
\end{align*}

which recovers the gradient update that LOLA does. 
\end{proof}

\begin{lemma}
\label{lemma:inner_approx}
For any arbitrary fixed $\theta^1$, the naive gradient step in LOLA: $\theta^2 - \Delta \theta^2 = \theta^2 - \lambda \nabla_{\theta^2} L^2 (\pi_{\theta^1}, \pi_{\theta^2})$ 
is equivalent to applying the proximal operator on a first-order approximation (around $\theta^2$) of the objective  $L^2 (\pi_{\theta^1}, \pi_{\theta^2})$.
\end{lemma}

The proof is similar to the one above for Lemma \ref{lemma:lola_is_prox}:

\begin{proof}
Let $\theta^1$ be fixed.  Let $f^2({\theta^2}) = L^2 (\pi_{\theta^{1}}, \pi_{\theta^2  } )$.

Apply the first-order Taylor approximation $\hat f_{\theta^2}^{2(1)}$ of the objective function (around $\theta^2$), as in the following:
\begin{align*}
\hat f_{\theta^2}^{2(1)} ({\theta^{2\prime\prime}}) &= f^2({\theta^2}) + \nabla_{\theta^2} f^2({\theta^2})^T ({\theta^{2\prime\prime}}-{\theta^2})
\end{align*}

Applying the proximal operator to the above, and using the results from Section \ref{background:prox:prelim} (with $x = \theta^{2\prime\prime}$ and $v = \theta^2$), we have:
\begin{align*}
\prox_{\hat f_{\theta^2}^{2(1)}} (\theta^{2}) &=\argmin\limits_{\theta^{2\prime\prime}} \left(\hat f_{\theta^2}^{2(1)} ({\theta^{2\prime\prime}})+ \frac{1}{2\lambda} || \theta^2 - \theta^{2\prime\prime}||_2^2 \right) \\ 
&= \theta^2 - \lambda \nabla_{\theta^2} f^2(\theta^2) = \theta^2 - \lambda \nabla_{\theta^2} L^2(\pi_{\theta^1}, \pi_{\theta^2})  = \theta^2 - \Delta \theta^2
\end{align*}
thus recovering the gradient update.
\end{proof}

\begin{theorem}
\label{thm:pola-lola}
POLA reduces to LOLA when replacing the divergence on policies with the L2 norm on policy parameters and using a first order approximation of both agents' objectives.
\end{theorem}
\begin{proof}
Recall the expression for POLA (Equations \ref{eq:POLA_outer} and \ref{eq:POLA_inner}):
\begin{align*}
\theta^{1\prime}(\theta^{1}, \theta^2) &= \argmin\limits_{\theta^{1\prime\prime}} \left(L^1 (\pi_{\theta^{1\prime\prime}}, \pi_{\theta^{2\prime}(\theta^{1\prime\prime}, \theta^2)}) + \beta_\outer D(\pi_{\theta^1} || \pi_{\theta^{1\prime\prime}}) \right) \\
\theta^{2\prime}(\theta^{1\prime\prime}, \theta^2) &=
\argmin\limits_{\theta^{2\prime\prime}} \left(L^2 (\pi_{\theta^{1\prime\prime}}, \pi_{\theta^{2\prime\prime}}  ) + \beta_\inner D(\pi_{\theta^2} || \pi_{\theta^{2\prime\prime}}) \right)
\end{align*}
Replace $L^2 (\pi_{\theta^{1\prime\prime}}, \pi_{\theta^{2\prime\prime}}  ) $ with a first order approximation of $L^2(\pi_{\theta^{1\prime\prime}}, \pi_{\theta^{2}} )$ around $\theta^2$, replace $D(\pi_{\theta^2} || \pi_{\theta^{2\prime\prime}})$ with $|| \theta^2 - \theta^{2\prime\prime}||_2^2$, and set $\beta_\inner = \frac{1}{2\lambda}$. Then by Lemma \ref{lemma:inner_approx}, $\theta^{2\prime}(\theta^{1\prime\prime}, \theta^2) = \theta^{2} - \Delta \theta^2(\theta^{1\prime\prime})$, resulting in the following equation (which is Outer POLA in Section \ref{section:pola_outer}):
\begin{align*}
\theta^{1\prime}(\theta^{1}, \theta^2) = \argmin\limits_{\theta^{1\prime\prime}} \left(L^1 (\pi_{\theta^{1\prime\prime}}, \pi_{\theta^2 - \Delta \theta^2(\theta^{1\prime\prime})}  ) + \beta_\outer D(\pi_{\theta^1} || \pi_{\theta^{1\prime\prime}}) \right)
\end{align*}
Replace $D(\pi_{\theta^1} || \pi_{\theta^{1\prime\prime}})$ with $|| \theta^1 - \theta^{1\prime\prime}||_2^2$, set $\beta_\outer = \frac{1}{2\lambda}$, and replace $L^1 (\pi_{\theta^{1\prime\prime}}, \pi_{\theta^2 - \Delta \theta^2(\theta^{1\prime\prime})}  ) $ with the first order approximation of $L^1(\pi_{\theta^1}, \pi_{\theta^2})$ around $\theta^1$, $\hat f_{\theta^1}^{1(1)} ({\theta^{1\prime\prime}})$:
\begin{align*}
\theta^{1\prime}(\theta^{1}, \theta^2)  = \argmin\limits_{\theta^{1\prime\prime}} \left(\hat f_{\theta^1}^{1(1)} ({\theta^{1\prime\prime}})+ \frac{1}{2\lambda} || \theta^1 - \theta^{1\prime\prime}||_2^2 \right)
\end{align*}
This is exactly $\prox_{\hat f_{\theta^1}^{1(1)}} (\theta^1)$, so by Lemma \ref{lemma:lola_is_prox}, this reduces to LOLA.
\end{proof}

\subsection{LOLA-DiCE}
\label{appendix:dice}

DiCE \citep{foerster2018dice} introduces an infinitely differentiable estimator for unbiased higher-order Monte Carlo gradient estimates. DiCE introduces a new operator $\epsdice{2}$ which operates on a set of stochastic nodes $\mathcal{W}$, where: 
\begin{equation}
\label{eq:dice_op}
\epsdice{2}(\mathcal{W}) = \exp(\tau - \bot(\tau)) \text{ and } 
\tau = \sum\limits_{w \in \mathcal{W}} \log p(w;\vtheta).
\end{equation}
$\bot$ is a stop-gradient operator (detach in Pytorch).
The loss for agent $1$, on a single rollout using policies ($\pi_{\theta^1}, \pi_{\theta^{2\prime}}$), with LOLA-DiCE is:
\begin{equation}
\label{eq:DiCE}
\mathcal{L}_{\epsdice{2}(\pi_{\theta^1}, \pi_{\theta^{2\prime}})}^1 = - \sum\limits_{t=0}^T \epsdice{2}({a}_{\leq t}) \gamma^t r_t^1
\end{equation}
where ${a}_{\leq t}$ denotes the set of actions all agents took at time step $t$ or earlier.
Algorithm \ref{algo:LOLA-DiCE} provides pseudo-code for LOLA-DiCE.

\begin{algorithm}
\caption{LOLA-DiCE 2-agent formulation: update for agent $1$}
\label{algo:LOLA-DiCE}
\begin{algorithmic}
\STATE {\bfseries Input:} Policy parameters $\boldsymbol\theta = \{\theta^1, \theta^2\}$, learning rates $\alpha, \eta$, number of inner steps $K$
\STATE Initialize: $\boldsymbol\theta' \gets \boldsymbol\theta$
\FOR {$k$ in $1...K$}
\STATE Rollout trajectories 
using policies $(\pi_{\theta^1}, \pi_{\theta^{2\prime}})$
\STATE $\theta^{2\prime} \gets \theta^{2\prime} - \eta \nabla_{\theta^{2\prime} } \mathcal{L}_{\epsdice{2}(\pi_{\theta^1}, \pi_{\theta^{2\prime}})}^2$
\ENDFOR
\STATE Rollout trajectories 
using policies $(\pi_{\theta^1}, \pi_{\theta^{2\prime}})$
\STATE $\theta^{1\prime} \gets \theta^1 - \alpha \nabla_{\theta^1} \mathcal{L}^1_{\epsdice{2}(\pi_{\theta^1}, \pi_{\theta^{2\prime}})}$
\STATE {\bfseries Output:} $\theta^{1\prime}$
\end{algorithmic}
\end{algorithm}

LOLA-DiCE \citep{foerster2018dice} replicated the original LOLA policy gradient results \citep{foerster2018learning} in a way that was more direct, efficient, and stable, supported by experiments with tabular policies in the one-step memory IPD.

\subsection{Loaded DiCE}
\label{appendix:loaded_dice}

\textit{Loaded DiCE} \citep{farquhar2019loaded} rewrites the DiCE objective (\ref{eq:DiCE}) as: 
\begin{equation}
\label{eq:loaded}
\mathcal{L}_{\epsdice{2}(\pi_{\theta^1}, \pi_{\theta^{2\prime}})}^1  =  - \sum\limits_{t=0}^T \gamma^t \left(\epsdice{2}({a}_{\leq t}) - \epsdice{2}({a}_{< t}) \right) \sum\limits_{l=t}^T \gamma^{l-t} r_l^1
\end{equation}
where ${a}_{< t}$ denotes the set of actions all agents took before time step $t$.

\citet{farquhar2019loaded} showed Equation \ref{eq:loaded} has the same gradients as Equation \ref{eq:DiCE}. 
$\sum\limits_{l=t}^T \gamma^{l-t} r_l^1$ in Equation \ref{eq:loaded} is then replaced with an advantage function: $A^1(s_t, a_t)$. Thus, \textit{loaded DiCE} incorporates a baseline for variance reduction with DiCE. 

In all our experiments with rollouts (Sections \ref{section:full_hist_ipd} and \ref{section:coin}), we use \textit{loaded DiCE} (Equation \ref{eq:loaded}) with generalized advantage estimation (see Appendix \ref{appendix:GAE} and \citet{schulman2015high})
for the baseline.

\subsection{Generalized Advantage Estimation}
\label{appendix:GAE}

Generalized advantage estimation \citep{schulman2015high}, introduces the following advantage estimator for time step $t$:
\begin{align*}
&\hat A_t^{GAE(\gamma, \lambda)} \triangleq (1 - \lambda) (\hat A_t^{(1)} + \lambda \hat A_t^{(2)} + \lambda^2 \hat A_t^{(3)} + ...) \\
&\hat A_t^{(k)} \triangleq \sum\limits_{l=0}^{k-1} \gamma^l \delta_{t+l}^V 
= -V(s_t) + r_t + \gamma r_{t+1} + ... + \gamma^{k-1} r_{t+k-1} + \gamma^k V(s_{t+k})
\end{align*}
When $\lambda = 0$, the advantage estimator is the one-step Bellman residual, whereas $\lambda = 1$ is equivalent to Monte Carlo estimation (in the finite horizon setting, it extrapolates the Monte Carlo estimation with the final state value). In our finite-horizon experiments, we use a finite sum:
\begin{align*}
&\hat A_t^{GAE(\gamma, \lambda, T)} \triangleq (1 - \lambda) \sum\limits_{t'=1}^T   \lambda^{t' - 1} \hat A_t^{(t')}
\end{align*}
which is the same as what \textit{loaded DiCE} does \citep{farquhar2019loaded}. 

When the value function $V$ is not completely accurate, the advantage estimator is biased. Using a lower value of $\lambda$ increases bias, but reduces variance.

\subsection{Training the Critic}

Advantage estimation requires a value function (critic) $V$. To train the critic, we minimize the mean squared error of the Monte Carlo return extended by the value in the final state:
\begin{align*}
\mathcal{L}_{critic} 
= 
\sum\limits_{t=0}^T ([-V(s_t)] + [r_t + \gamma r_{t+1} + ... + \gamma^{T-t-1} r_{T-1} + \gamma^{T-t} V(s_{T})]) ^2 
\end{align*}
averaged over samples in the batch.

\subsection{POLA-DiCE with Repeated Training on the Same Samples}
\label{appendix:pola_dice_repeat_train}

POLA-DiCE can require a lot of environment rollouts. One idea for improving sample efficiency and reducing training time is developing a modification like  \citet{schulman2017proximal} or \citet{kakade2002approximately} that allows for repeated training on the same set of samples.

We define a new operator similar to (\ref{eq:dice_op}) from LOLA-DiCE:
\begin{equation}
\epsdice{3}(\mathcal{W}) = \exp(\tau' - \bot(\tau')) \text{ and } 
\tau' = \sum\limits_{w \in \mathcal{W}} p(w;\vtheta) / p(w;\vtheta_{old}) 
\end{equation}
where $\vtheta_{old}$ denotes policy parameters used on the first rollout (step $k = 1$ for the inner loop and step $m = 1$ for the outer loop). 

Our new loss for agent $1$ is:
$$\mathcal{L}_{\epsdice{3}(\pi_{\theta^1}, \pi_{\theta^{2\prime}})}^1 = \sum\limits_{t=0}^T \left(\epsdice{3}({a}_{\leq t}) - \epsdice{3}({a}_{< t}) \right) \gamma^t  A^1(s_t, a_t)$$
This is the \textit{loaded DiCE} formulation (\ref{eq:loaded}), but with probability ratios instead of log probabilities, which lets us make multiple updates on the same batch of rollouts.
We provide pseudo-code in Algorithm \ref{algo:pola_dice_ppo}; we call this POLA-DiCE-PPO.

\begin{algorithm}
\caption{POLA-DiCE-PPO 2-agent formulation: update for agent $1$}
\label{algo:pola_dice_ppo}
\begin{algorithmic}
\STATE {\bfseries Input:} Policy parameters $\theta^1, \theta^2$, learning rates $\alpha_1, \alpha_2$, penalty hyperparameters $\beta_\inner, \beta_\outer$, number of outer steps $M$ and inner steps $K$
\STATE Initialize: $\theta^{1\prime\prime} \gets \theta^1$
\FOR {$m$ in $1...M$}
\STATE Initialize: $\theta^{2\prime\prime} \gets \theta^2$
\FOR {$k$ in $1...K$}
\IF {k = 1}
\STATE Rollout trajectories 
using policies $(\pi_{\theta^{1\prime\prime}}, \pi_{\theta^{2\prime\prime}})$. Save states $s^{\text{in}}_{\leq T}$ from the trajectories
\ENDIF
\STATE $\theta^{2\prime\prime} \gets \theta^{2\prime\prime} - \alpha_2 \nabla_{\theta^{2\prime\prime} } 
( \mathcal{L}_{\epsdice{3}(\pi_{\theta^{1\prime\prime}}, \pi_{\theta^{2\prime\prime}})}^2  + \beta_\inner D(\pi_{\theta^2}, \pi_{\theta^{2\prime\prime}} | s^{\text{in}}_{\leq T}) )
$
\ENDFOR
\STATE Rollout trajectories with states $s^{\text{out}}_{\leq T}$
using policies $(\pi_{\theta^{1\prime\prime}}, \pi_{\theta^{2\prime\prime}})$
\STATE $\theta^{1\prime\prime} \gets \theta^{1\prime\prime} - \alpha_1 \nabla_{\theta^{1\prime\prime}} (\mathcal{L}^1_{\epsdice{3}(\pi_{\theta^{1\prime\prime}}, \pi_{\theta^{2\prime\prime}})} + \beta_\outer D(\pi_{\theta^1}, \pi_{\theta^{1\prime\prime}} | s^{\text{out}}_{\leq T}))$
\ENDFOR
\STATE {\bfseries Output:} $\theta^{1\prime\prime}$
\end{algorithmic}
\end{algorithm}
On the first update step ($m=1$ and $k=1$), $\theta_\old = \bot(\theta)$, so
$\nabla_{\theta^{2\prime\prime} } \mathcal{L}_{\epsdice{3}(\pi_{\theta^{1\prime\prime}}, \pi_{\theta^{2\prime\prime}})}^2 = \nabla_{\theta^{2\prime\prime} } \mathcal{L}_{\epsdice{2}(\pi_{\theta^{1\prime\prime}}, \pi_{\theta^{2\prime\prime}})}^2$ 
and 
$\nabla_{\theta^{1\prime\prime}} \mathcal{L}^1_{\epsdice{3}(\pi_{\theta^{1\prime\prime}}, \pi_{\theta^{2\prime\prime}})} = \nabla_{\theta^{1\prime\prime}} \mathcal{L}^1_{\epsdice{2}(\pi_{\theta^{1\prime\prime}}, \pi_{\theta^{2\prime\prime}})}$.
Thus, when inner steps $K = 1$, outer steps $M = 1$, and $\beta_\inner, \beta_\outer = 0$, POLA-DiCE-PPO is equivalent to LOLA-DiCE. 

Algorithm \ref{algo:pola_dice_ppo} repeats training only on the inner loop. A similar repeated training procedure can be used on the outer loop, for even greater sample efficiency. However, this works only if the higher order gradients calculated from the initial rollout remain accurate as the policy changes.

We present this section only in the Appendix because we empirically found a larger number of outer steps to be more important than a larger number of inner steps, and that repeatedly training on the same samples for the outer loop resulted in learning reciprocity-based cooperation significantly less consistently. In future work, modifications such as periodic rollouts every $x < M$ outer steps may provide improved sample efficiency with less deterioration in performance.

This objective can also be clipped, like the clipped version of PPO \citep{schulman2017proximal}; future work could explore this and compare to the KL penalty version.

\subsection{POLA N-Agent Formulation}
\label{appendix:pola_N}

Consider again agent 1's perspective. Agent 1 solves for the following policy:
\begin{equation}
\label{eq:POLA_outer_N}
\theta^{1\prime}(\theta^1, \theta^2, ..., \theta^N) = 
\argmin\limits_{\theta^{1\prime\prime}} \left(L^1 (\pi_{\theta^{1\prime\prime}}, \pi_{\theta^{2\prime}(\theta^{1\prime\prime}, \theta^2, ..., \theta^N)}, ... , \pi_{\theta^{N\prime}(\theta^{1\prime\prime}, \theta^2, ... , \theta^N)} ) + \beta_\outer D(\pi_{\theta^1} || \pi_{\theta^{1\prime\prime}}) \right)
\end{equation}
where $D(\pi_{\theta^i} || \pi_{\theta^{i\prime\prime}})$ is again shorthand for a general divergence defined on policies.
For $\theta^{2\prime}, ... , \theta^{N\prime}$ in Equation \ref{eq:POLA_outer_N}, we choose the following proximal updates:
\begin{equation*}
\label{eq:POLA_inner_N}
    \begin{aligned}
    \theta^{2\prime}(\theta^{1\prime\prime}, \theta^2, ... , \theta^N) &=
    \argmin\limits_{\theta^{2\prime\prime}} \left(L^2 (\pi_{\theta^{1\prime\prime}}, \pi_{\theta^{2\prime\prime}}, \pi_{\theta^{3}} ... ,  \pi_{\theta^{N}} ) + \beta_\inner D(\pi_{\theta^2} || \pi_{\theta^{2\prime\prime}}) \right) \\
    & \vdots \\
    \theta^{N\prime}(\theta^{1\prime\prime}, \theta^2, ... , \theta^N) &=
    \argmin\limits_{\theta^{N\prime\prime}} \left(L^N (\pi_{\theta^{1\prime\prime}}, \pi_{\theta^{2}}, ... , \pi_{\theta^{N-1}} ,  \pi_{\theta^{N\prime\prime}} ) + \beta_\inner D(\pi_{\theta^N} || \pi_{\theta^{N\prime\prime}}) \right)
    \end{aligned}
\end{equation*}

In short, agent 1 assumes all other agents $i$ find the argmin of their loss, assuming the policies of other agents $j \neq i$ are fixed (using agent 1’s updated policy $\theta^{1\prime\prime}$ and the original policies of all other agents).

\subsection{POLA-DiCE N-Agent Formulation}
\label{appendix:pola_N_dice}

Let $\pi_{\theta^{-1, -i}} \triangleq \{\pi_{\theta^2}, ..., \pi_{\theta^N} \} \setminus \pi_{\theta^i} $ be shorthand for all of the policies except those of agent 1 and agent $i$. Algorithm \ref{algo:pola_dice_N} provides an N-agent formulation of POLA-DiCE. The idea is similar to Appendix \ref{appendix:pola_N}; in the inner loop, each agent updates assuming other agents' policies are fixed.

\begin{algorithm}
\caption{POLA-DiCE N-agent formulation: update for agent $1$}
\label{algo:pola_dice_N}
\begin{algorithmic}
\STATE {\bfseries Input:} Policy parameters $\theta^1, \theta^2, ..., \theta^N$, learning rates $\alpha_1, \alpha_2, ... , \alpha_N$, penalty hyperparameters $\beta_{\inner}^2, ... , \beta_{\inner}^N$, $\beta_{\outer}$, number of outer steps $M$ and inner steps $K$
\STATE Initialize: $\theta^{1\prime\prime} \gets \theta^1$
\FOR {$m$ in $1...M$}
\FOR {$i$ in $2...N$} 
\STATE Initialize: $\theta^{i\prime\prime} \gets \theta^i$
\FOR {$k$ in $1...K$}
\STATE Rollout trajectories with states $s^{\text{in}}_{\leq T}$
using policies $(\pi_{\theta^{1\prime\prime}}, \pi_{\theta^{i\prime\prime}}, \pi_{\theta^{-1, -i}})$
\STATE $\theta^{i\prime\prime} \gets \theta^{i\prime\prime} - \alpha_i \nabla_{\theta^{i\prime\prime} } 
( \mathcal{L}_{\epsdice{2}(\pi_{\theta^{1\prime\prime}}, \pi_{\theta^{i\prime\prime}}, \pi_{\theta^{-1, -i}})}^i  + \beta_{\inner}^i D(\pi_{\theta^i}, \pi_{\theta^{i\prime\prime}} | s^{\text{in}}_{\leq T}) )
$
\ENDFOR
\ENDFOR
\STATE Rollout trajectories with states $s^{\text{out}}_{\leq T}$
using policies $(\pi_{\theta^{1\prime\prime}}, \pi_{\theta^{2\prime\prime}}, ..., \pi_{\theta^{N\prime\prime}})$
\STATE $\theta^{1\prime\prime} \gets \theta^{1\prime\prime} - \alpha_1 \nabla_{\theta^{1\prime\prime}} (\mathcal{L}^1_{\epsdice{2}(\pi_{\theta^{1\prime\prime}}, \pi_{\theta^{2\prime\prime}}, ..., \pi_{\theta^{N\prime\prime}})} + \beta_{\outer} D(\pi_{\theta^1}, \pi_{\theta^{1\prime\prime}} | s^{\text{out}}_{\leq T}))$
\ENDFOR
\STATE {\bfseries Output:} $\theta^{1\prime\prime}$
\end{algorithmic}
\end{algorithm}

\section{Appendix: Experiment and Hyperparameter Details}
\label{appendix:expmt_hparams}

\subsection{One-Step Memory IPD}

\subsubsection{Exact Loss Calculation}
\label{appendix:exact_loss_ipd}

With one step of memory, we can directly build the transition probability matrix $\mathcal{P}$ and starting state distribution $\textbf{p}_0$ given all agents' policies. \citet{foerster2018learning} in their Appendix A.2 derived the exact
loss: $L^i = - \textbf{p}_0^T (I - \gamma \mathcal{P})^{-1} R^i$.
This can then be directly used for gradient updates.

\subsubsection{Example Showing how the Cooperation Factor Recovers the IPD Reward Structure}
\label{appendix:cf_details}

Below we show that $f = 4 / 3$ recovers the IPD from \citet{foerster2018learning}:

Let $M$ be the number of agents who cooperate, $D(M)$ be the payoff to each defector, and $C(M)$ be the payoff to each cooperator. Then $C(2) = 1/3$, $C(1) = -1/3$, $D(1) = 2/3$, $D(0) = 0$, and we have the payoff matrix:
\begin{center}
\begin{tabular}{c|c|c}
P1/P2 & C & D \\
 \hline
C & (1/3, 1/3) & (-1/3, 2/3)   \\
\hline
D & (2/3, -1/3)  & (0, 0)
\end{tabular}  
\end{center}
Multiplying all rewards by 3 (equivalent to scaling the learning rate), subtracting 2 from all rewards (which preserves the ordering of policies by expected reward, and is equivalent under an accurately learned value function baseline), then gives the payoff matrix:
\begin{center}
\begin{tabular}{c|c|c}
P1/P2 & C & D \\
 \hline
C & (-1, -1) & (-3, 0)   \\
\hline
D & (0, -3)  & (-2, -2)
\end{tabular}
\end{center}

which is the 2-player IPD with the reward structure given in \citet{foerster2018learning}.

\subsubsection{Function Approximation Setup}
\label{appendix:ipd-one-funcapprox-state}

With function approximation, our state representation is a one-hot vector with 3 dimensions (defect, cooperate, and start state) for each agent's past action. Thus, with two agents, each agent's input is two 3-d one-hot vectors, which we flatten to a single 6-d vector.
We use this representation as it has size $\Theta(N)$, whereas
a single one-hot vector over all possible combinations of actions has size $\Theta(2^N)$; our representation is conducive to future experiments with large $N$.

\subsubsection{Hyperparameter Settings}
\label{appendix:ipd1_hparams}

Typical neural network weight initializations (e.g. Gaussian) produce policies that are close to random (cooperating with probability close to 0.5 in each state). Our policy probability initializations are close to random throughout; this helps provide comparable results between tabular and neural network policies.
Naive learning can find TFT if initialized sufficiently close to it, but never finds TFT with policies initialized close to random.

Empirically, for policies initialized close to random and for sufficiently large $\eta$, LOLA always updates the policy toward cooperating \textit{iff} the opponent last cooperated. Thus, with sufficiently large learning rate $\alpha$, LOLA finds TFT in the tabular setting (for $f > 1$). However, LOLA no longer finds TFT with large learning rates when the policy parameterization is a neural network function approximator or a transformed tabular policy (Figure \ref{fig:all_comparison}).

For Figure \ref{fig:all_comparison}, we show results from the best set of hyperparameters (highest \% TFT found). We tuned inner and outer learning rates ($\eta, \alpha$) for LOLA using a greedy heuristic; we increased or decreased values until we no longer got better results. For \textit{outer POLA}, we tuned $\eta$ and $\beta_\outer$, and did less tuning on the outer learning rate, which matters only for convergence speed and stability. For the exact set of hyperparameters, see: \url{https://github.com/Silent-Zebra/POLA}.

\subsubsection{Additional Detailed IPD Results}
\label{appendix:ipd-one-detailed_results}

\begin{table}
  \caption{Average Policies (Probability of Cooperation by State) Learned in IPD for Various Policy Parameterizations}
  \label{table:detailed_avg_policies}
  \centering
  \begin{tabular}{lllllll}
    \toprule
    Algorithm   & Parameterization     & DD     & DC  & CD & CC & Start \\
     & (Self \& Opponent)  \\
    \midrule

    \multicolumn{7}{c}{Contribution Factor 1.1}                   \\
    \midrule
    Naive Learning & Tabular & 0.00 & 0.06 & 0.06 & 0.16 & 0.01 \\
    LOLA & Tabular & 0.00 & 1.00 & 0.00 & 1.00 & 1.00 \\
    LOLA & Neural net & 0.00 & 0.01 & 0.00 & 0.03 & 0.02 \\
    LOLA & Pre-condition & 0.00 & 0.00 & 0.97 & 0.00 & 0.00 \\
    Outer POLA &  Tabular & 0.00 & 0.97 & 0.45 & 1.00 & 0.94 \\
    Outer POLA & Neural net & 0.00 & 0.94 & 0.49 & 1.00 & 0.94 \\
    Outer POLA & Pre-condition & 0.01 & 0.29 & 0.26 & 0.87 & 0.60 \\
    \midrule
    \multicolumn{7}{c}{Contribution Factor 1.25}                   \\
    \midrule
    Naive Learning & Tabular  & 0.00 & 0.07 & 0.07 & 0.18 & 0.02 \\ 
    LOLA & Tabular & 0.00 & 1.00 & 0.00 & 1.00 & 1.00 \\
    LOLA & Neural net & 0.00 & 0.02 & 0.01 & 0.06 & 0.02 \\
    LOLA & Pre-condition & 0.00 & 0.00 & 0.96 & 0.00 & 0.00 \\
    Outer POLA & Tabular & 0.08 & 0.97 & 0.18 & 1.00 & 0.95 \\
    Outer POLA & Neural net & 0.01 & 0.87 & 0.47 & 1.00 & 0.85 \\
    Outer POLA & Pre-condition & 0.12 & 0.97 & 0.23 & 1.00 & 0.77 \\
    \midrule
    \multicolumn{7}{c}{Contribution Factor 1.33}                   \\
    \midrule
    Naive Learning & Tabular & 0.00 & 0.07 & 0.07 & 0.19 & 0.02 \\ 
    LOLA & Tabular & 0.00 & 1.00 & 0.00 & 1.00 & 1.00 \\
    LOLA & Neural net & 0.03 & 0.35 & 0.06 & 0.41 & 0.15 \\
    LOLA & Pre-condition & 0.00 & 0.00 & 0.96 & 0.00 & 0.00 \\
    Outer POLA & Tabular & 0.13 & 0.96 & 0.08 & 1.00 & 0.94 \\
    Outer POLA & Neural net & 0.02 & 0.85 & 0.45 & 0.99 & 0.68 \\
    Outer POLA & Pre-condition & 0.18 & 0.99 & 0.30 & 1.00 & 0.76 \\
    \midrule
    \multicolumn{7}{c}{Contribution Factor 1.4}                   \\
    \midrule
    Naive Learning & Tabular & 0.00 & 0.07 & 0.07 & 0.20 & 0.02 \\
    LOLA & Tabular & 0.00 & 1.00 & 0.00 & 1.00 & 1.00 \\
    LOLA & Neural net & 0.25 & 0.87 & 0.37 & 0.88 & 0.65 \\
    LOLA & Pre-condition & 0.00 & 0.00 & 0.96 & 0.00 & 0.00 \\
    Outer POLA & Tabular & 0.13 & 0.96 & 0.12 & 1.00 & 0.95 \\
    Outer POLA & Neural net & 0.02 & 0.87 & 0.44 & 0.99 & 0.76 \\
    Outer POLA & Pre-condition & 0.21 & 0.98 & 0.33 & 1.00 & 0.75 \\
    \midrule
    \multicolumn{7}{c}{Contribution Factor 1.6}                   \\
    \midrule
    Naive Learning & Tabular & 0.00 & 0.09 & 0.09 & 0.23 & 0.05 \\
    LOLA & Tabular  & 0.00 & 1.00 & 0.00 & 1.00 & 1.00 \\
    LOLA & Neural net & 0.15 & 0.98 & 0.41 & 0.97 & 0.67 \\
    LOLA & Pre-condition & 0.10 & 0.10 & 0.87 & 0.10 & 0.10 \\
    Outer POLA & Tabular & 0.05 & 0.90 & 0.07 & 1.00 & 0.91 \\
    Outer POLA & Neural net & 0.03 & 0.95 & 0.28 & 0.99 & 0.87 \\
    Outer POLA & Pre-condition & 0.08 & 0.98 & 0.20 & 1.00 & 1.00 \\
    \bottomrule
  \end{tabular}
\end{table}

In Table \ref{table:detailed_avg_policies}, we show the average probability of cooperation in each state of the IPD for each of the algorithms and contribution factors shown in Figure \ref{fig:all_comparison}. Here, DD denotes both agents last defected, DC denotes the agent last defected while the opponent last cooperated, CD denotes the agent last cooperated while the opponent last defected, CC denotes both agents last cooperated, and Start is the starting state. Numbers are averaged over 20 runs and averaged over both agents. \textit{Outer POLA} learns reciprocity-based cooperation across contribution factor and policy parameterization settings much more consistently than LOLA.

\subsubsection{IPD Experiments with Varying Opponent Model Parameterizations}
\label{appendix:ipd-om}

We revisit the IPD with opponent modeling, considering agents with tabular policies but varying parameterizations of opponent models: tabular, neural network function approximation, and pre-conditioned tabular models. These parameterizations are the same as described previously except for the pre-conditioned tabular model, for which we now use:
\begin{align*}
\textbf{Q}^1 = 
\begin{psmallmatrix}
1 & 0 & 0 & 0 & 0 \\
-2 & 1 & 0 & 0 & 0 \\
-2 & 0 & 1  & 0 & 0 \\
-2 & 0  & 0  & 1 & 0\\
-2 & 0  & 0 & 0 & 1
\end{psmallmatrix}
,
\textbf{Q}^2 = 
\begin{psmallmatrix}
1 & 0 & 0 & 0 & 0 \\
-2 & 1 & 0 & 0 & 0 \\
-2 & 0 & 1  & 0 & 0 \\
-2 & 0  & 0  & 1 & 0\\
-2 & 0  & 0 & 0 & 1
\end{psmallmatrix}
\end{align*}
This set of transformations again only changes basis and lets us illustrate the difference between LOLA and POLA. 
We show results in Table \ref{table:detailed_avg_policies_om}. LOLA fails to learn reciprocity-based cooperation for the pre-conditioned tabular opponent model, whereas POLA learns reciprocity-based cooperation across all parameterizations.

\begin{table}
  \caption{Average Policies (Probability of Cooperation by State) Learned in IPD for Various Opponent Model Policy Parameterizations}
  \label{table:detailed_avg_policies_om}
  \centering
  \begin{tabular}{lllllll}
    \toprule
    Algorithm   & Parameterization     & DD     & DC  & CD & CC & Start \\
     & (Opponent Model)  \\
    \midrule

    \multicolumn{7}{c}{Contribution Factor 1.33}                   \\
    \midrule
    LOLA & Tabular & 0.02 & 0.99 & 0.17 & 1.00 & 0.93 \\
    LOLA & Neural net & 0.00 & 0.99 & 0.03 & 1.00 & 0.97 \\
    LOLA & Pre-condition & 0.00 & 0.10 & 0.07 & 0.23 & 0.08 \\
    POLA &  Tabular & 0.01 & 0.97 & 0.05 & 1.00 & 0.97 \\
    POLA & Neural net & 0.01 & 0.97 & 0.05 & 1.00 & 0.97 \\
    POLA & Pre-condition & 0.12 & 0.94 & 0.02 & 1.00 & 0.85 \\
    \bottomrule
  \end{tabular}
\end{table}

LOLA learns reciprocity-based cooperation well even with a neural net opponent policy. We believe this is because the inner player policy update (usually towards defecting in each state) is less sensitive to policy parameterization than the outer player policy update which requires second order gradients and learning reciprocity. 

The version of POLA we use here is similar to POLA-DiCE (Algorithm \ref{algo:pola_dice}) except with exact losses and the uniform distribution KL penalty; we provide pseudocode in Algorithm \ref{algo:pola_approx}. 

\begin{algorithm}
\caption{POLA direct approximation 2-agent formulation: update for agent $1$}
\label{algo:pola_approx}
\begin{algorithmic}
\STATE {\bfseries Input:} Policy parameters $\theta^1, \theta^2$, learning rates $\alpha_1, \alpha_2$, penalty hyperparameters $\beta_\inner, \beta_\outer$, number of outer steps $M$ and inner steps $K$
\STATE Initialize: $\theta^{1\prime\prime} \gets \theta^1$
\FOR {$m$ in $1...M$}
\STATE Initialize: $\theta^{2\prime\prime} \gets \theta^2$
\FOR {$k$ in $1...K$}
\STATE $\theta^{2\prime\prime} \gets \theta^{2\prime\prime} - \alpha_2 \nabla_{\theta^{2\prime\prime} } 
( L^2 (\pi_{\theta^{1\prime\prime}}, \pi_{\theta^{2\prime\prime}})  + \beta_{\inner} (\E_{s \sim U(\S)} [D_{KL}(\pi_{\theta^2}(s) || \pi_{\theta^{2\prime\prime}} (s))]) )
$
\ENDFOR
\STATE $\theta^{1\prime\prime} \gets \theta^{1\prime\prime} - \alpha_1 \nabla_{\theta^{1\prime\prime}} ({L}^1(\pi_{\theta^{1\prime\prime}}, \pi_{\theta^{2\prime\prime}}) + \beta_{\outer} (\E_{s \sim U(\S)} [D_{KL}(\pi_{\theta^1}(s) || \pi_{\theta^{1\prime\prime}} (s))]))$
\ENDFOR
\STATE {\bfseries Output:} $\theta^{1\prime\prime}$
\end{algorithmic}
\end{algorithm}

We use 100-200 inner steps for POLA and 1 outer step; since we are only changing the opponent model parameterization, invariance on the inner loop is most important. This is a fairly small number of steps which does not achieve full invariance, causing results to vary slightly across parameterizations. We assume unlimited batch size for learning the opponent model, which is equivalent to directly learning from the policy, to separate the effects of noise from opponent modeling. Detailed hyperparameters are available in the codebase (\url{https://github.com/Silent-Zebra/POLA}).

\subsection{IPD Full History Details}
\label{appendix:full_hist_ipd}

We use the state representation as discussed in Appendix \ref{appendix:ipd-one-funcapprox-state}. Agents condition actions on the entire state history up to the current time step. Policies are parameterized by a fully connected input layer with 64 hidden units, a ReLU nonlinearity, a GRU cell, and a fully connected output layer. We use a batch size of 2000 (parallel environment rollouts), discount rate $\gamma = 0.96$, and rollout for $T=50$ steps in the environment. For both LOLA-DiCE and POLA-DiCE, we use a simple gradient step on the inner loop and the Adam optimizer with default betas on the outer loop, as LOLA-DiCE \citep{foerster2018dice} does.

For all algorithms, we use \textit{loaded DiCE} (Appendix \ref{appendix:loaded_dice}) with GAE (Appendix \ref{appendix:GAE}) with $\lambda = 1$, an outer learning rate of 0.003, and a learning rate for the critic (value function) of 0.0005. For LOLA-DiCE, we use an inner learning rate of 0.05, but we also tried the values: [0.005, 0.015, 0.02, 0.03, 0.07, 0.1, 0.2] and got similar or worse results. We also tried a few settings with lower and higher outer learning rates and value function learning rates, and 2 inner steps; none learned reciprocity-based cooperation more consistently.

For POLA-DiCE, we use 2 inner steps and 200 outer steps, with $\beta_\inner = 10$ and $\beta_\outer = 100$, and an inner learning rate of 0.005. Results are not very sensitive to $\beta_\outer$; we got similar results with $\beta_\outer = 30$ and $\beta_\outer = 200$. We did not extensively tune hyperparameters for POLA-DiCE in this setting, so it is likely that similar results can be reproduced with fewer outer steps. Results are more sensitive to the inner learning rate and $\beta_\inner$, but we expect using more inner steps would lessen this sensitivity. One challenge with taking more inner steps is the memory requirement, which forces a tradeoff with batch size. We update the critic after each policy update on both the inner and outer loop (for the corresponding agent or opponent model). 

To learn the opponent model in POLA-OM, we use 200 environment rollouts (of 2000 batch size each) between each set of POLA updates. The opponent model architecture is the same as the agent's. We learn the opponent's value function in the same way as the agent's own value function, but with the reward and states of the other agent. We use learning rates of 0.005 for the policy model and 0.0005 for the value model. For other hyperparameters, we use the same settings for POLA-DiCE with and without opponent modeling.

In Figure \ref{fig:gru_ipd} we choose the number of outer steps as the x-axis because the outer steps are policy updates that are actually made; inner step updates are not saved, and are used only in the gradient calculation of outer step updates. Each inner step currently requires an environment rollout, though this can be mitigated in future work (e.g. Appendix \ref{appendix:pola_dice_repeat_train}), another reason why we consider outer steps more representative of sample efficiency. Strictly comparing environment rollouts would horizontally stretch the lines for POLA-DiCE relative to LOLA-DiCE by a factor of 1.5 (3 environment rollouts per outer step for POLA-DiCE vs. 2 for LOLA-DiCE); this would not change the conclusions drawn in the paper.

\subsection{Coin Game Details}
\label{appendix:coin}

Our environment implementation adheres to Figure 3 in \citet{foerster2018learning} where two agents that step on the same coin at the same time both collect the coin. In previous experiments, splitting the coin 50-50 between agents gave very similar results.

The LOLA results in Figure \ref{fig:coin} cannot be directly compared with those in \citet{foerster2018learning} for several reasons. Our implementation of the coin game environment fixes bugs such as ties always being broken in favour of the red agent (see: \href{https://github.com/alshedivat/lola/issues/9}{https://github.com/alshedivat/lola/issues/9}), which make the original results irreproducible. We use LOLA-DiCE instead of the original LOLA-PG formulation. We rollout for fewer steps to reduce computation time and memory requirements. Our policy parameterization has more hidden units and uses a GRU, which should make it more expressive than the originally used RNN, and may also make optimization more difficult.

Same as the IPD with full history, 
our GRUs use 64 hidden units with a fully connected input layer with ReLU nonlinearity and a linear output layer. We use a batch size of 2000 (parallel environment rollouts), discount rate $\gamma = 0.96$, and rollout for $T=50$ steps in the environment. For both LOLA-DiCE and POLA-DiCE, we use a simple gradient step on the inner loop and the Adam optimizer with default betas on the outer loop, as LOLA-DiCE \citep{foerster2018dice} does.

For all algorithms, we use \textit{loaded DiCE} (Appendix \ref{appendix:loaded_dice}) with GAE (Appendix \ref{appendix:GAE}) with $\lambda = 1$, an outer learning rate of 0.003, and a learning rate for the critic (value function) of 0.0005. 
For LOLA-DiCE, we use an inner learning rate of 0.003, but we also tried the values: [0.001, 0.005, 0.01, 0.02, 0.05, 0.1, 0.2] and got similar results. We also tried a few settings with lower and higher outer learning rates and value function learning rates, and 2 inner steps; none learned reciprocity-based cooperation more consistently.

For POLA-DiCE, we use 2 inner steps and 200 outer steps, with $\beta_\inner = 5$ and $\beta_\outer = 150$, and an inner learning rate of 0.02.
We update the critic after each policy update on both the inner and outer loop (for the corresponding agent or opponent model).

To learn the opponent model in POLA-OM, we use 200 environment rollouts (of 2000 batch size each) between each set of POLA updates. The opponent model architecture is the same as the agent's. We learn the opponent's value function in the same way as the agent's own value function, but with the reward and states of the other agent. We use learning rates of 0.005 for the policy model and 0.0005 for the value model. For POLA-OM, we use 4 inner steps with $\beta_\inner = 10$ and an inner learning rate of 0.01; other hyperparameters are the same as POLA-DiCE without opponent modeling. The additional inner steps and higher $\beta_\inner$ provide greater invariance to the opponent model, which helps in the coin game. We tried a few settings with even more inner steps but those did not learn reciprocity-based cooperation more consistently. We believe this may be due to memory constraints forcing smaller batch sizes (and thus more noise from environment rollouts) with more inner steps. 

In Figure \ref{fig:coin} we choose the number of outer steps as the x-axis because the outer steps are policy updates that are actually made; inner step updates are not saved, and are used only in the gradient calculation of outer step updates. Each inner step currently requires an environment rollout, though this can be mitigated in future work (e.g. Appendix \ref{appendix:pola_dice_repeat_train}), another reason why we consider outer steps more representative of sample efficiency. Strictly comparing environment rollouts would horizontally stretch the lines for POLA-DiCE relative to LOLA-DiCE by a factor of 1.5 (3 environment rollouts per outer step for POLA-DiCE vs. 2 for LOLA-DiCE); this would not change the conclusions drawn in the paper.

\subsection{Code Details}
\label{appendix:code_details}

Parts of code were adapted from \url{https://github.com/alexis-jacq/LOLA_DiCE} \citep{foerster2018dice} and \url{https://github.com/aletcher/stable-opponent-shaping} \citep{letcher2018stable}. Both use the MIT license, which grants permission free of charge for subsequent use, modification, and distribution. 

\subsection{Compute Usage}
\label{appendix:compute}

For the IPD with one-step memory (Section \ref{section:hist_one_ipd}), experiments were run on CPUs provided free of charge by Google Colaboratory. Most experiments required only a small amount of compute (taking minutes to run).

For the IPD with full history (Section \ref{section:full_hist_ipd}) and the coin game (Section \ref{section:coin}), experiments were run on GPUs on an internal cluster. GPUs were either NVIDIA Tesla T4 or NVIDIA Tesla P100. Coin game experiments took around 1 full day (24 hours) to run for 1 seed on 1 GPU, whereas the IPD with full history experiments took around 8-10 hours for each seed.

The total amount of compute used, including during the experimentation phase, was significantly higher than that used for Figures \ref{fig:gru_ipd} and \ref{fig:coin}. 

\section{Societal Impact}
\label{section:impact}

We do not anticipate any immediate societal impact from this work; at the time of writing, there is no direct real world application. That said, we hope this work helps produce socially beneficial outcomes when autonomous learning agents interact, which is critical for future real-world deployment. However, while opponent shaping helps in the social dilemma settings we tested, it could cause undesirable consequences in other settings. For example, pricing algorithms learning reciprocity-based cooperation would be tantamount to collusion. In such cases, POLA could learn undesirable behaviour in a way that is invariant to policy parameterization.
Overall though, we expect the potential positive impact of our work to outweigh the potential negative impact.

\end{document}